\documentclass[final,12pt]{colt2022} 

\usepackage{changes}

\title[(Nearly) Optimal Private Linear Regression for Sub-Gaussian Data via Adaptive Clipping]{(Nearly) Optimal Private Linear Regression via Adaptive Clipping}
\usepackage{times}
\usepackage{microtype}
\usepackage{booktabs} 
\usepackage{float}
\usepackage[normalem]{ulem}

\usepackage{mathtools,algorithm,algorithmic}

\author{%
 \Name{Prateek Varshney} \Email{vprateek@google.com}
 \AND
 \Name{Abhradeep Thakurta} \Email{athakurta@google.com}
  \AND
 \Name{Prateek Jain} \Email{prajain@google.com}
}

\usepackage{stmaryrd}
\usepackage{graphicx}
\usepackage{titlesec}
\usepackage{hyperref}
\hypersetup{
    colorlinks=true,
    linkcolor=blue,
    filecolor=magenta,      
    urlcolor=blue,
    pdfpagemode=FullScreen,
    }
\usepackage{xspace}
\usepackage{natbib}
\usepackage{framed}
\usepackage{float}
\usepackage{colortbl}

\usepackage{colortbl}
\definecolor{DarkGreen}   {rgb}{0,0.5,0}
\definecolor{DarkBlue}    {rgb}{0,0.0,0.5}
\definecolor{LightGray}   {rgb}{0.8,0.8,0.8}

\def\LatinUpper{A,B,C,D,E,F,G,H,I,J,K,L,M,N,O,P,Q,R,S,T,U,V,W,X,Y,Z}
\def\LatinLower{a,b,c,d,e,f,g,h,i,j,k,l,m,n,o,p,q,r,s,t,u,v,w,x,y,z}

\newcommand{\genCal}[1]{\expandafter\newcommand\csname c#1\endcsname{{\mathcal #1}}}
\makeatletter
\@for\i:=\LatinUpper\do{%
	\expandafter\genCal\i
}
\makeatother

\newcommand{\genBb}[1]{\expandafter\newcommand\csname b#1\endcsname{{\mathbb #1}}}
\makeatletter
\@for\i:=\LatinUpper\do{%
	\expandafter\genBb\i
}
\makeatother

\newcommand{\genFk}[1]{\expandafter\newcommand\csname k#1\endcsname{{\mathfrak #1}}}
\makeatletter
\@for\i:=\LatinUpper\do{%
	\expandafter\genFk\i
}
\makeatother

\newcommand{\genFkl}[1]{\expandafter\newcommand\csname k#1\endcsname{{\mathfrak #1}}}
\makeatletter
\@for\i:=\LatinLower\do{%
	\expandafter\genFkl\i
}
\makeatother

\renewcommand{\vec}[1]{{\mathbf{#1}}}
\newcommand{\genLatinVec}[1]{\expandafter\newcommand\csname v#1\endcsname{{\vec #1}}}
\makeatletter
\@for\i:=\LatinLower\do{%
	\expandafter\genLatinVec\i
}
\makeatother

\newcommand{\genLatinVecU}[1]{\expandafter\newcommand\csname v#1\endcsname{{\vec #1}}}
\makeatletter
 \@for\i:=\LatinUpper\do{%
 	\expandafter\genLatinVecU\i
}
\makeatother

\def\mydefgreek#1{\expandafter\def\csname v#1\endcsname{\text{\boldmath$\mathbf{\csname #1\endcsname}$}}}
\def\mydefallgreek#1{\ifx\mydefallgreek#1\else\mydefgreek{#1}%
   \lowercase{\mydefgreek{#1}}\expandafter\mydefallgreek\fi}
\mydefallgreek {alpha}{beta}{gamma}{delta}{epsilon}{varepsilon}{zeta}{eta}{theta}{vartheta}{iota}{kappa}{lambda}{mu}{nu}{xi}{omicron}{pi}{rho}{varrho}{sigma}{tau}{upsilon}{phi}{varphi}{chi}{psi}{omega}\mydefallgreek

\def\mydefugreek#1{\expandafter\def\csname v#1\endcsname{\text{\boldmath$\mathbf{\csname #1\endcsname}$}}}
\def\mydefallugreek#1{\ifx\mydefallugreek#1\else\mydefugreek{#1}%
   \lowercase{\mydefugreek{#1}}\expandafter\mydefallugreek\fi}
\mydefallugreek{Gamma}{Delta}{Theta}{Lambda}{Xi}{Sigma}{Upsilon}{Phi}{Psi}{Omega}\mydefallugreek

\newcommand{\vzero}{\vec{0}}

\newcommand{\bc}[1]{\left\{{#1}\right\}}
\newcommand{\br}[1]{\left({#1}\right)}
\newcommand{\bs}[1]{\left[{#1}\right]}

\newcommand{\ip}[2]{\left\langle{#1},{#2}\right\rangle}

\newcommand{\softO}[1]{\widetilde{\cal O}\br{{#1}}}

\newcommand{\softOm}[1]{\widetilde\Omega\br{{#1}}}

\newcommand{\E}[1]{\bE\bs{{#1}}}
\newcommand{\Ee}[2]{\underset{#1}\bE\bs{{#2}}}

\newcommand{\argmin}{\mathop{\arg\min}}

\newcommand\numberthis{\addtocounter{equation}{1}\tag{\theequation}}

\newcommand{\dpnoise}{{\alpha}}

\newcommand{\wo}{\vw^*}

\newcommand{\privc}{c_{\sf priv}}
\newcommand{\dpssgd}{{\sf DP-SSGD}\xspace}
\newcommand{\dpambssgd}{{\sf DP-AMBSSGD}\xspace}
\newcommand{\dpstat}{{\sf DP-STAT}\xspace}

\newcommand{\ltwo}[1]{\left\|#1\right\|_2}
\renewcommand{\epsilon}{\varepsilon}
\newcommand{\eps}{\varepsilon}

\newcommand{\calA}{\ensuremath{\mathcal{A}}}

\newcommand{\calC}{\ensuremath{\mathcal{C}}}
\newcommand{\calD}{\ensuremath{\mathcal{D}}}

\newcommand{\calN}{\ensuremath{\mathcal{N}}}

\newcommand{\calS}{\ensuremath{\mathcal{S}}}

\newcommand{\mypar}[1]{\smallskip
	\noindent{\textbf{{#1}:}}}

\newcommand{\bfw}{\ensuremath{\mathbf{w}}}

\newcommand{\hp}{\beta}
\newcommand{\fieldn}{z}
\newcommand{\vfieldn}{\vz}

\newcommand{\s}[1]{\mathsf{#1}}

\usepackage[normalem]{ulem}
\newcommand{\stkout}[1]{\ifmmode\text{\sout{\ensuremath{#1}}}\else\sout{#1}\fi}

\begin{document}

\maketitle

\begin{abstract}
We study the problem of differentially private {\em linear regression} where each  data point is sampled from a fixed sub-Gaussian style distribution. We propose and analyze a one-pass  mini-batch stochastic gradient descent method (DP-AMBSSGD) where points in each iteration are sampled {\em without}  replacement. Noise is added for DP but the noise standard deviation is estimated online. Compared to existing $(\epsilon, \delta)$-DP techniques which have  sub-optimal error bounds, DP-AMBSSGD is able to provide nearly optimal error bounds in terms of key parameters like dimensionality $d$, number of points $N$, and the standard deviation $\sigma$ of the noise  in observations. For example, when the $d$-dimensional covariates are sampled i.i.d. from the normal distribution, then the excess error of DP-AMBSSGD due to {\em privacy}  is $\frac{\sigma^2 d}{N}(1+\frac{d}{\epsilon^2 N})$, i.e., the error is meaningful when number of samples $N= \Omega(d \log d)$ which is the standard operative regime for linear regression. In contrast, error bounds for existing efficient methods in this setting are:  $\cO\big(\frac{d^3}{\epsilon^2 N^2}\big)$, even for $\sigma=0$. That is, for constant $\epsilon$, the existing techniques require $N=\Omega(d\sqrt{d})$ to provide a non-trivial result.  
\end{abstract}

\section{Introduction}
 Machine Learning (ML) models are known to be susceptible to leaks of sensitive private information of individual points in the training data. In fact, this risk is non-trivial even for problems as simple and canonical as linear regression \citep{hsukakade,dieuleveut2016harder,jain2018parallelizing}. 
 
 Several existing works have studied privacy preserving linear regression and more generally,  private convex optimization with differential privacy as the privacy notion  \citep{chaudhuri2011differentially,kifer2012private,BST14,song2013stochastic,BassilyFTT19,mcmahan2017learning, WLKCJN17,andrew2021differentially,feldman2019private,bassily2020stability,song2020characterizing,iyengar2019towards}. While tight upper and lower bounds are known for generic problem classes (e.g., convex Lipschitz losses, and strongly convex losses)~\citep{BST14,BassilyFTT19}, surprisingly, the popular special instance of \emph{linear regression} is much less mapped~\citep{smith2017interaction,sheffet2019old,liu2021differential,Cai21}.

 In this paper, we study the problem of DP linear regression (DP-LR) and provide (nearly) optimal excess population risk guarantees under standard assumptions, i.e., the data is sampled i.i.d. from a sub-Gaussian distribution. In particular, we study the following problem: Consider a dataset $D=\{(\vx_0, y_0), \dots, (\vx_{N-1}, y_{N-1})\}$ drawn i.i.d. from some fixed distribution $\cD$ with $\vH=\E{\vx\vx^\s{T}}$. Without loss of generality, $y=\langle \vx, \wo\rangle+\fieldn$ where $\E{\fieldn\cdot  \vx}=0$ and $\E{\fieldn^2}=\sigma^2$. The goal in DP-LR is to output a model $\vw_{\sf priv}$ that preserves privacy and also  approximately minimizes the excess population risk  i.e., $\cL(\vw_{\sf priv})-\cL(\wo)$, where,  $$\cL(\vw)= \frac{1}{2}\mathbb{E}_{(\vx,y)\sim \cD}[(y-\langle \vx, \vw\rangle)^2].$$
 
  Now, even for the above set of assumptions, the existing {\em polynomial time} methods require $N\geq d\sqrt{d}$ for them to be non-vacuous (for constant $\fieldn$). Recent work by \cite{liu2021differential} indeed obtains strong error rates similar to our results, however their proposed method is exponential in $d$ and $N$. In contrast, our methods are linear in both $d$ and $N$. 
 
 The key issue in most of the existing works is that the global Lipschitz constant of $\cL$ scales as $\langle \vx, \vw-\wo\rangle \leq \|\vx\|_2 \|\vw-\wo\|_2$ which implies that sensitivity of gradient would already have $O(d/N)$ term that leads to rate of $d\sqrt{d}/N$ as $d$-dimensional noise vector with $O(d/N)$ standard deviation has norm of the order: $d\sqrt{d}/N$. This issue recurs in most of the existing works, as intuitively they try to ensure that {\em all} the directions of covariance are privacy preserving, while our goal is to ensure only direction along $\wo$ is differentially private. 
 
 In this work, we propose a method DP-SSGD (DP-Shuffled SGD) to get around this challenge by using a one-pass noisy SGD method that samples points {\em without} replacement and uses tail averaging. As we take only one-pass over the shuffled data, we can ensure that the $t^\s{th}$ iterate $\vw_t$ is completely independent of the sampled $\vx_t$ and hence, $\vw_t$ are independent in each iteration. This implies that $\langle \vx, \vw\rangle$ can have a significantly tighter bounded  for sub-Gaussian style distributions. In fact, we can separate the loss using standard bias-variance decomposition of the loss \citep{jain2018parallelizing}. This decomposition along with amplification by shuffling \citep{feldman2021hiding}, ensures that the excess risk $\cL(\vw_{\sf priv})-\cL(\wo)$ of the method is bounded by $\widetilde{\cO}(\sigma^2 d/N + (1+\sigma^2)d^2/{\epsilon^2 N^2})$\footnote{In $\widetilde \bfw(\cdot), \widetilde O(\cdot)$, and $\widetilde \Omega(\cdot)$ we hide polylog factors in $N$ and $(1/\delta)$.}. However, the proposed method has three issues: a) the sample complexity of the method is $N = \widetilde{\Omega}(d^2)$, b) requires $\epsilon \leq 1/\sqrt{N}$, and c) the second term is sub-optimal w.r.t. $\sigma$. 
 
 To address the above mentioned weaknesses with DP-SSGD's analysis, we modify DP-SSGD to obtain a new method -- DP-AMBSSGD (DP-Adaptive-Mini-Batch-Shuffled-SGD) -- which divides the data into mini-batches and runs one pass of shuffled mini-batch SGD but where the noise is set adaptively according to the excess error in each iteration. For this method, we can obtain excess risk of the form $\widetilde{O}(\sigma^2 d/N + \sigma^2 d^2/{\epsilon^2 N^2})$ with sample complexity $N\geq d \log^2 d$.  That is, the proposed method is efficient with time complexity of only $O(Nd)$ while still ensuring nearly optimal error rate that matches the lower bound in \cite{Cai21}, up to constants and condition number factors. 
 
Below we provide an informal version of this result in a simplified setting. 
\begin{theorem}[Informal result]
\label{thm:informalMain}
Let $\vx_i \sim \cN(\vzero,\vI)$ and $y_i=\langle \vx_i, \wo\rangle + \fieldn_i$, for $\wo\in \mathbb{R}^d$ and $\fieldn_i \sim \sigma N(0,1)$. Then, there exists a method for DP-LR that guarantees $(\epsilon, \delta)$-DP, has time-complexity $O(Nd)$ and has excess risk bounded by (with probability $\geq 1-1/N^{100}$ over randomness in data and algorithm) : 
$$\cL(\vw_{\sf priv})-\cL(\wo)\leq 8\sigma^2\frac{d}{N}+\softO{\sigma^2 \frac{d^2}{\epsilon^2 N^2}}, \text{ if }  N\geq d \log^3 d.$$ 
\end{theorem}
As mentioned above, the excess risk bound matches the lower bound by \cite{Cai21}. Furthermore, the algorithm is  optimal in terms of time complexity. Finally, in the same setting, existing methods \citep{Cai21} have an excess risk bound of $\sigma^2 \frac{d}{N} + \frac{(d\|\wo\|_{\vH}^2+\sigma^2)\cdot d^2}{\epsilon^2 N^2}$ if $N\geq d^{3/2}$. Note that the above bound is significantly sub-optimal for $\sigma\leq \sqrt{d}$ which is a fairly practical setting. Also note that the bound in \cite{Cai21} is stated only in the setting of $\sigma\geq \sqrt{d}$. Hence the claim of matching upper-lower bounds by \cite{Cai21} hold only for specific setting of $\sigma$. In contrast, our result holds for all range of $\sigma$, especially for theoretically and empirically important settings of small $\sigma$. Furthermore, while \cite{Wang:2017} also claims optimal dimension dependence in excess risk bound, their results (and the ones cited within) hide factors corresponding to the Lipschitz constant ($L$) of the least squares function. As mentioned above, $L$ can scale as $\|\vx\|_2\|\wo\|_2$, and thus would have an additional factor of $\sqrt{d}$ for the standard Gaussian distribution. Clearly, when $\epsilon > \sqrt{d}/\sqrt{N}$, the first term dominates and hence one can assume that the privacy is for ”free” asymptotically. Otherwise, when $\epsilon = c\sqrt{d}/\sqrt{N}$ (with $c < 1$), the asymptotic bound translates to $\widetilde{O}(d/cN)$, i.e., DP does a sample size reduction from $N \rightarrow cN$. This is common for DP-ERM. Finally, our results lead to optimal variance error, similar to non-private versions. Moreover, while the privacy related error term (second error term in Theorem ~\ref{thm:ambsgd-utility}) is optimal in $d, N, \epsilon$, it has an additional dependence on $\kappa$--the conditional number of $\vH=\E{\vx\vx^\s{T}}$. For first order methods, we do expect such a factor to be present due to composition theorems, but it is not clear what is the optimal dependence on $\kappa$ for the privacy related error term. We leave further investigation into this factor for future work.  

\subsection{Our Algorithms}
\label{sec:ourContrib}

\mypar{DP-Shuffled SGD (DP-SSGD)} Our first algorithm is a fairly straightforward adaptation of standard DP-SGD~\citep{song2013stochastic,BST14}, where we make a single pass over the dataset with single data sample mini-batches, along with appropriate clipping of the gradients\footnote{Clipping of a vector, corresponds to scaling down a vector to a fixed $\ell_2$-norm $\zeta$, if its norm exceeds $\zeta$.} and addition of Gaussian noise. Finally, we output the average of last $N/2$ models produced by the algorithm. To have strong privacy guarantee, we randomly shuffle the dataset $D$ prior to running the algorithm. Using the by-now standard tool of privacy amplification by shuffling~\citep{erlingsson2019private,feldman2021hiding}, we amplify the overall privacy guarantee (as compared to the analysis for the unshuffled dataset) by a factor of $\approx 1/\sqrt{N}$. With this tool in hand, it is not hard to show that one needs to add noise $\widetilde O\left(\frac{\zeta}{\epsilon\sqrt N}\right)$ to satisfy $(\epsilon,\delta)$-DP as long as $\epsilon\leq \frac{1}{\sqrt N}$, where $\zeta$ is the clipping norm of the gradients~\citep{DP-DL}. Furthermore, using the bias-variance decomposition analysis of the linear regression loss~\citep{jain2018parallelizing}, one can show that setting $\zeta=\widetilde O(\sigma\sqrt d)$ ensures a excess risk of $\widetilde{O}(\sigma^2 d/N + (1+\sigma^2)d^2/{\epsilon^2 N^2})$. For small values of the privacy parameter $\epsilon$, this bound is optimal assuming the standard deviation of the inherent noise is $\Theta(1)$. However, the proposed method has three major drawbacks:  a) requires $\epsilon \leq 1/\sqrt{N}$, b)  the sample complexity of the method is $N\geq \Omega(d^2)$ because of the upper bound on $\epsilon$, and c) the second term in the error is  sub-optimal w.r.t. $\sigma$. To that end, we propose Algorithm DP-AMBSSGD (DP-Adaptive-Mini-Batch-Shuffled-SGD) to address each of these issues.

\mypar{DP-Adaptive-Mini-Batch-Shuffled-SGD (DP-AMBSSGD)} This algorithm is based on the traditional DP-SGD framework (Differentially Private Stochastic Gradient Descent)~\citep{song2013stochastic,BST14,DP-DL}. For a mini-batch of data samples $\widehat D_t=\{(\vx_i,y_i)\}_{i=0}^{b-1}$, DP-AMBSSGD does a state update of the form $\vw_{t+1}\leftarrow \vw_t-\eta\Big(\frac{1}{b}\sum\limits_{(\vx,y)\in \widehat D_t}\text{clip}_{\zeta_t}(\vx\cdot(\ip{\vx}{\vw_t} - y))+\calN\left(0,\frac{4\alpha^2\zeta_t^2}{b^2}\right)\Big)$, where $\alpha$ is called the noise multiplier (which is a function of only the number of iterations and the privacy parameters $(\epsilon,\delta)$), $\zeta_t$ is called the clipping norm, and $\text{clip}_\zeta(\vnu) = \vnu\cdot\min\Big\{1, \frac{\zeta}{\|\vnu\|_2}\Big\}$. The key idea behind the design of DP-AMBSSGD is to adaptively choose the clipping norm $\zeta_t$ s.t. a ``overwhelming majority'' of the gradients do not get clipped at any time step $t$. Observe that as one gets closer to the true model $\vw^*$, the norm of the gradient of the linear regression loss at an data point $(\vx,y)$ approximately scales as $\mathbb{E}\left[\ltwo{\vx}\right]\cdot\sigma$, where $\sigma$ is the standard deviation of the inherent noise.  Since, $\sigma$  and $\mathbb{E}\left[\ltwo{\vx}\right]$ are unknown, one cannot use this as a proxy for the clipping norm. We circumvent this issue as follows, at each time step $t$ we estimate a threshold $\zeta_t$ (via a DP algorithm) s.t. the norms of $99\%$ of the gradients in the mini-batch fall below $\zeta_t$. We then use $\zeta_t$ as the clipping norm in the standard DP-SGD algorithm. We call the combined algorithm DP-AMBSSGD. As the training progresses and $\vw_t$ gets closer to the true model $\vw^*$, concentration bound on $\zeta_t$ provably gets smaller and hence the noise added to ensure DP also progressively reduces. Using the same bias-variance analysis of DP-SSGD above will result in Theorem~\ref{thm:ambsgd-utility} that addresses all the three issues in Algorithm DP-SSGD. The privacy analysis follows from standard adaptive composition properties of DP~\citep{bun2016concentrated}.

It is worth mentioning that such an iterative localization/clipping idea was also used in the context of mean estimation~\citep{biswas2020coinpress}. \citep{kamath18} also previously used similar ideas for covariance estimation while \citep{feldman2019private} used localization in a DP-SCO context.

It is also worth mentioning that the idea of \emph{adaptive clipping} in DP-SGD has been empirically demonstrated to be beneficial~\citep{andrew2021differentially}, but to our knowledge, DP-AMBSSGD is the first to formally demonstrate its advantage over fixed clipping.

\subsection{Paper Organization}

We present related works to our work in next subsection. We then define the problem and required notations in Section~\ref{sec:prob}. Next, we present the one-pass SGD method  and it's analysis in Section~\ref{sec:ssgd}. We then present our main result and the algorithm in Section~\ref{sec:ambssgd}. Finally, we conclude with a discussion of our results and several exciting future directions.

\subsection{Other Related Works}
\label{sec:related}

Differentially private empirical risk minimization (DP-ERM) is probably one of the most well-studied subfields of differential privacy~\citep{chaudhuri2011differentially,kifer2012private,BST14,song2013stochastic,DP-DL,BassilyFTT19,mcmahan2017learning, WLKCJN17, iyengar2019towards, pichapati2019adaclip,andrew2021differentially,feldman2019private,bassily2020stability,song2020characterizing}. 
The objective there is to solve the following problem while preserving DP: $\argmin\limits_{\bfw\in\calC}\frac{1}{N}\sum\limits_{i=0}^{N-1}\ell(\bfw;d_i)$, where $D=\{d_0,\ldots,d_{N-1}\}$ is the dataset. When the loss function $\ell(\cdot;\cdot)$ is convex, tight excess population risk bound of $L\ltwo{\calC}\cdot\widetilde\bfw\left(\frac{1}{\sqrt N}+\frac{\sqrt d}{\epsilon N}\right)$ is known under $(\epsilon,\delta)$-DP~\citep{BassilyFTT19,bassily2020stability}, where $L$ is the $\ell_2$-Lipschitz constant of the loss function $\ell(\cdot;\cdot)$ and $\calC$ is the diameter of the \emph{convex} constraint set. The algorithm that achieves this bound is essentially a variant of differentially private stochastic gradient descent (DP-SGD)~\citep{song2013stochastic,BST14}, or objective perturbation~\citep{chaudhuri2011differentially,kifer2012private}. For the problem of linear regression considered in this paper, it is not hard to observe that $L\approx d$ when the constraint set $\ltwo{\calC}=O(1)$. This means the sample complexity for the excess population bound above is $N=\widetilde\Omega(d^2)$. One can take a direct approach, and perturb the sufficient statistics for linear regression~\citep{smith2017interaction,sheffet2019old} and achieve a sample complexity of $N=\widetilde\Omega(d^{3/2})$. In this work we show that if the feature vectors are drawn from $d$-dimensional standard Normal distribution, we can get an excess population risk of $\widetilde O\left(\frac{d}{N}+\frac{d^2}{\epsilon^2N^2}\right)$, which translates to a sample complexity of $N=\widetilde\Omega(d)$. For DP sparse principal component analysis (PCA)~\citep{chien2021private}, and for mean estimation~\citep{brown2021covariance,liu2021differential}, similar improvements on the sample complexity were observed in prior works under ``Gaussian'' like assumptions. To the best of our knowledge, the techniques in these papers are complementary to ours, and cannot be used for the problem at hand.

A crucial component of our algorithms is that it makes a single pass over the data, after the dataset has been uniformly permuted. Our algorithms are thus an addition to the growing list of single pass algorithms~\citep{feldman2019private,bassily2020stability} that achieve optimal population risk under $(\epsilon,\delta)$-DP. 


\section{Problem Definition and Preliminaries}\label{sec:prob}
Let $D=\{(\vx_0, y_0),\ (\vx_1, y_1),\ \dots,\ (\vx_{N-1}, y_{N-1})\}$ be the input dataset with each point $(\vx_i, y_i)$ sampled i.i.d. from a distribution $\cD$. 
Given $D$, the goal is to learn $\wo$ that minimizes the Least Squares Regression (LSR) problem: 
\begin{align*}
\mathsf{(LR):}\quad\quad    \wo=\argmin_\vw \cL(\vw) := \frac{1}{2}  \Ee{(\vx, y) \sim \cD}{(y - \langle \vx, \vw \rangle)^2}. \numberthis \label{eq:lsobj}
\end{align*}
Furthermore, the goal is to ensure that $\wo$ is privacy preserving (see definitions below). Without loss of generality, we have: 
$$y_i=\langle \vx_i, \wo \rangle + \fieldn_i,\text{ where }\E{\fieldn_i \vx_i}=0,\  \text{ and }   \E{\fieldn_i^2}=\sigma^2.$$

For simplicity, we assume that $\fieldn$ is {\em independent} of $\vx$. Note that our results hold for general case when $\fieldn$ is dependent on $\vx$ with appropriately defined notation but for simplicity of exposition, we focus only the independent case.  

\noindent We now discuss our notations and then later specify assumptions for the input distribution $\cD$.

{\bf Notations}: We use boldface lower letters such as $\vx$, $\vw$ for vectors, boldface capital letters (e.g. $\vA$, $\vH$) for matrices and normal script font letters (e.g. $\cM$, $\cT$) for tensors. Let $\|\vA\|_2$ denote the spectral norm of $\vA$ and $\|\vv\|_2$ denote the Euclidean norm of $\vv$. Unless specified otherwise, $\|\vA\|=\|\vA\|_2$ and $\|\vv\|=\|\vv\|_2$. Also,  $\|\vx\|_\vA:=\|\vA^{1/2} \vx\|$ and $\|\vM\|_\vA := \|\vA^{-1/2} \vM \vA^{-1/2}\|_2$.

Let $\vH := \E{\vx\vx^\s{T}}$ be the second-moment  matrix and let $\cM:=\E{\vx {\otimes}\vx {\otimes}\vx {\otimes}\vx}$ be the fourth moment tensor. $\cT\otimes \va$ denotes outer product of $\cT$ with $\va$. Let  $\mu$ be the smallest eigenvalue of $\vH$.  Let $R_x$ be the smallest $R$ which satisfies $\E{\|\vx\|_2^2 \vx\vx^\s{T}}\leq R^2 \E{\vx\vx^\s{T}}$. This implies $\E{\|\vx\|_2^2}\leq R_x^2$. Finally, let $\vSigma = \E{(y - \langle \vx, \wo\rangle)^2\vx\vx^\s{T}}$ be the error ``covariance" matrix. 

\begin{definition}[$(\vH,K_2,a,\beta)$-Tail]
Let both $\vH$ and $\cM$ exist and are finite. Also, let $\mu>0$ and $\kappa=\|\vH\|_2/\mu<\infty$. A random vector $\vx$ satisfies $(\vH,K_2,a,\beta)$-Tail if the the following holds:
\begin{itemize}
    \item $\exists a>0$ s.t. with probability $\geq 1-\hp$, \begin{equation}\|\vx\|_2^2 \leq \mathbb{E}\left[\ltwo{\vx}^2\right]
\cdot\log^{2a} (1/\hp)\leq R_x^2\cdot\log^{2a} (1/\hp),\label{eq:xnorm}\end{equation}
\item We have, 
\begin{equation}\max_{\vv, \|\vv\|=1} \E{\exp\left(\left(\frac{|\langle \vx, \vv\rangle |^2}{K_2^2 \|\vH\|_2}\right)^{1/2a}\right)} \leq 1.\label{eq:assump_vx}\end{equation}
That is, from \eqref{eq:xnorm} and \eqref{eq:assump_vx}, for any fixed  $\vv$, w.p. $\geq 1-\hp$: \begin{equation}(\langle \vx, \vv\rangle)^2 \leq K_2^2 \|\vH\|_2 \|\vv\|^2 \log^{2a} (1/\hp).\label{eq:vxbound}\end{equation}
\end{itemize}
\label{def:subGAssumption}
\end{definition}


\noindent We assume the distribution on feature vectors $\vx\sim\calD$ satisfies $(\vH,K_2,a,\hp)$-Tail, and the inherent noise $\fieldn$ satisfies $(\sigma^2\vI,K_2,a,\hp)$-Tail. Note that if both $\vx$ and $\fieldn$ are sampled from a normal distribution then the above requirements are satisfied for $a=1/2$. Beyond sub-Gaussian, notice that Definition~\ref{def:subGAssumption} captures a more general classes of distributions like subExponential (when we set $a=1$).

\noindent {\bf Notation of Privacy}: In this work, we operate in the standard setting of differential privacy (DP)~\citep{DMNS,ODO} in the replacement model. In this paper, we assume $\eps=O(1)$, and $\delta=\cO(1/{\sf poly}(N))$, where $N$ is the data set size. The formal definition is below. 

\begin{definition}[Differential privacy \citep{DMNS, ODO}] A randomized  algorithm $\calA$ is $(\eps,\delta)$-differentially private if, for any pair of datasets $D$ and $D'$ differ in exactly one data point (i.e., one data point is present in one set and absent in another for both $D$ and $D'$), and for all events $\calS$ in the output range of $\calA$, we have 
$$\Pr[\calA(D)\in \calS] \leq e^{\eps} \cdot \Pr[\calA(D')\in \calS] +\delta,$$
where the probability is taken over the random coin flips of $\calA$. 
\label{def:diiffP}
\end{definition}

In the privacy analysis we use zCDP (zero-Concentrated Differential Privacy~\citep{bun2016concentrated}) for privacy accounting. However, we state the final privacy guarantee always in terms of $(\eps,\delta)$-DP. For completeness purposes we define zCDP below:

\begin{definition}[zCDP~\cite{bun2016concentrated}]
A randomized algorithm $\calA$ is $\rho$-zCDP if for any pair of data sets $D$ and $D'$ that differ in one record, we have $D_{\alpha}\left(\calA(D)||\calA(D')\right)\leq \rho \alpha$ for all $\alpha > 1$, where $D_{\alpha}$ is the R\'enyi divergence of order $\alpha$.
\end{definition}


\section{Private Linear Regression}
Recall that the goal is to find the optimal linear regression parameter vector while preserving differential privacy of each individual point $(\vx_i, y_i)$. As mentioned earlier, existing techniques like DP-SGD \citep{DP-DL} suffer from error rates that require $N \geq d \sqrt{d}$ to be non-vacuous. Furthermore, it seems that the issue is fundamental to the method and the excess risk bounds are  unlikely to improve by better analysis. For example, in DP-SGD, one needs to add noise proportional to the gradient of a randomly sampled point {\em with replacement}. Now, gradient of the loss term corresponding to the $i^\s{th}$ data point is: $\nabla_\vw L_i(\vw)=(\langle \vx_i, \vw \rangle - y_i) \vx_i$. As $\vx_i$ is sampled with replacement, $\vw$ can be dependent on $\vx_i$ (e.g. if the same index $i$ is sampled twice in a row). This implies, that $\langle \vx_i, \vw \rangle$ might be as large as $\|\vx_i\|_2 \|\vw\|_2$. That is, for the simple setting of $\vx_i \sim \cN(\vzero,\vI)$, $\ltwo{\nabla_\vw L_i(\vw)} \approx d$ which implies $O(d/N)$ noise would be required in each iteration of DP-SGD which leads to an additional error of $d\sqrt{d}/N$ in the estimation of $\wo$. 

The above observation provides a key motivation for our algorithm. In particular, it is clear that sampling with replacement can lead to challenging dependencies between the current iterate $\vw$ and the sampled point. Instead, if we sample points {\em without} replacement and take only {\em one} pass over the data, then at each step $\vw$ is {\em independent} of $\vx_i$, which implies significantly better bound on the per step gradient. Naturally, the error bound can suffer if we take only {\em one} pass over the data. But by using some recent results \citep{jain2018parallelizing,sgd} along with techniques like tail averaging, we provide nearly {\em optimal} excess risk bounds.   

In the section below, we use the above intuition to propose the DP-Shuffled SGD (\dpssgd) method and provide analysis of the algorithm. However, our result for the proposed method holds only for $\epsilon \leq 1/\sqrt{N}$, and has sub-optimal sample complexity. To alleviate these concerns, we modify the proposed method using mini-batches of data which provide a nearly optimal error bound without significant restrictions on privacy budget $\epsilon$.  

\subsection{DP-Shuffled SGD}\label{sec:ssgd}
In this section, we present \dpssgd (Algorithm~\ref{alg:dp-sgd}) which is primarily based on DP-SGD with random shuffling i.e. sampling without replacement. That is, we first randomly permute the data points (Step 2). We then take one pass over the dataset, and update $\vw$ by gradient descent with the loss term corresponding to each data point. To ensure differential privacy, we clip the gradient norm to at most $\zeta$ and add Gaussian noise with standard deviation $2\eta\zeta\alpha$ (Step 6).  Finally, we output the average of the last half of the iterates (Step 8).  
\begin{algorithm}[H]
\caption{\dpssgd: DP-Shuffled SGD}\label{alg:dp-sgd}
\begin{algorithmic}[1]
    \STATE{\bfseries Input:} Samples: $\{(\vx_i, y_i)\}_{i=0}^{N-1}$, Clipping Norm: $\zeta$, DP Noise Multiplier: $\alpha$, Learning Rate: $\eta$
    \STATE Randomly permute $\{(\vx_i, y_i)\}_{i=0}^{N-1}$
    \STATE Initialize $\vw_0 \gets \vzero$
    \FOR{$t=0 \dots N-1$}
        \STATE Sample $\vg_t \sim \cN(\vzero, \vI_{d\times d})$
        {\STATE $\vw_{t+1} \leftarrow \vw_{t} - \eta \Big(\text{ clip}_\zeta(\vx_t(\langle \vx_t, \vw_t\rangle - y_t)) + \alpha \cdot 2\zeta \cdot \vg_t\Big)$ where $\text{clip}_\zeta(\vnu) = \vnu\cdot\min\Big\{1, \frac{\zeta}{\|\vnu\|_2}\Big\}$\label{line:abdf}}
    \ENDFOR
    \STATE{\bfseries Return:} $\overline{\vw} := \frac{2}{N}\sum_{t=N/2+1}^{N}\vw_{t}$
\end{algorithmic}
\end{algorithm}

We first show that Algorithm~\ref{alg:dp-sgd} is $(\epsilon, \delta)$-DP. The proof follows using ~\citep[Theorem 3.8]{feldman2021hiding} which analyzes DP composition with "randomly shuffled" data. 
\begin{theorem}\label{thm:sgd-privacy}
Suppose we apply \dpssgd (Algorithm~\ref{alg:dp-sgd}) on $N$ input samples with noise multiplier $\dpnoise =  \Omega\Big(\frac{\log(N/\delta)}{\varepsilon\sqrt{N}}\Big)$. Then it satisfies $(\varepsilon,\delta)$-differential privacy with $\varepsilon=\cO\Big(\sqrt{\frac{\log(N/\delta)}{N}}\Big)$.
\end{theorem}
{\bf Remarks:} Note that the privacy budget $\epsilon$ is restricted to be $\epsilon\leq 1/\sqrt{N}$ because $\epsilon$ has to be less than $1$ in each iteration to satisfy technical requirements of \cite{feldman2021hiding}. 
Also note that the privacy guarantee holds {\em irrespective} of data distribution $\cD$. Finally, clipping parameter $\zeta$  as well as learning rate $\eta$ requires knowledge of certain key quantities about the dataset and $\cD$. In practice, we would need to estimate these quantities while preserving privacy. There are standard approaches for such hyperparameter tuning~\citep{liu2019private,papernot2021hyperparameter}. However, being consistent with prior literature on DP optimization, and to highlight the key points of our approach, throughout the paper we assume a priori knowledge of the  hyperparameters. Further, the largest and smallest eigenvalue of Hessian can be estimated in a differentially private manner using standard sensitivity analysis with output perturbation, since the sensitivity of eigenvalues is constant if the feature vector have constant norm. Under our sample complexity assumption, the rates remain the same up to constants. Now, we present the error bounds for our method. 
\begin{theorem}\label{thm:sgd-utility}
Let $D=\{(\vx_i, y_i)\}_{i=0}^{N-1}$ be sampled i.i.d. with $\vx_i \sim \cD$ satisfying $(\vH,K_2,a,\hp)$-Tail, and the distribution of the inherent noise $\fieldn$ satisfies $(\sigma^2,K_2,a,\hp)$-Tail with $\beta=\frac{1}{N^{100}}$. Let $\wo$ be the optimal solution to the population least squares problem, $\kappa$ be the condition number of the covariance matrix $\vH$, and $\mathbb{E}\left[\ltwo{\vx}^2\right]\leq R_x^2$ (refer Section~\ref{sec:prob} for details about the notations and assumptions).\\ 
Initialize parameters in \dpssgd as follows:  stepsize $\eta = \min\Big\{\frac{1}{2R_x^2},\frac{c_1}{\log^{4a+2} N}\cdot  \frac{1}{K_2^2R_x^2 \kappa}\cdot \frac{1}{d\alpha_{\epsilon, \delta, N}^2}, \\\frac{c_2}{\log^{2a+2} N}\cdot \frac{1}{R_x^2}\Big\}$, where $c_1, c_2>0$ are global constants, noise multiplier $\alpha=\alpha_{\epsilon, \delta, N} = \cO\Big(\frac{\log(N/\delta)}{\varepsilon\sqrt{N}}\Big)$ and clipping threshold $\zeta = 4K_2R_x\cdot \log^{2a}N\cdot (\sqrt{\|\vH\|_2}\|\wo\|+\sqrt{\kappa} \sigma)$. Then, the output $\overline{\vw}$ of \dpssgd  achieves the following excess risk w.p. $\geq 1-1/N^{100}$ over randomness in data and algorithm:
\begin{equation*}
\cL(\overline{\vw}) - \cL(\vw^*) 
    \leq 
     \widetilde{\cO}\Big(\frac{\sigma^2 d}{N}+ \frac{  d^2\log^2(N/\delta)}{ \epsilon^2 N^2} \cdot K_2^2 \kappa^2 ( \|\wo\|^2_{\vH}+ \sigma^2) \Big),
\end{equation*}
assuming $N\geq \Omega(\frac{\kappa d \log ^{2a+1} (\kappa d)\log (1/\delta)}{\epsilon})$.
\end{theorem}
{\bf Remarks:}\\
{\bf i)} The error bound consists of two terms. The first term corresponds to the error due to noise in observations, which is minimax optimal up to a factor of $2$. The second term corresponds to the error incurred due to the noise added for preserving privacy. \\
{\bf ii)} Note that the error term corresponding to privacy budget for existing methods \citep{Cai21,BassilyFTT19,Wang:2017} is $\approx \frac{d^2 R_x^2\|\wo\|^2}{\epsilon^2 N^2}$. That is, the error term is $O(d)$ worse than our result. Furthermore, as analyzed in Corollary~\ref{cor:gauss_ssgd_utility},  for standard setting where each $\vx_i \sim \cN(\vzero, \vI)$ and  $y_i=\langle \vx_i, \wo\rangle$, existing results imply error bound of $\frac{d^3}{\epsilon^2 N^2}$ in contrast to our bound of $\frac{d^2}{\epsilon^2 N^2}$. \\
{\bf iii)} Theorem requires $\epsilon$ to scale as $1/\sqrt{N}$ which is restrictive in practice. Furthermore, under such restriction, the sample complexity of $N$ is $N\geq \kappa^2d^2$ which is also sub-optimal. Finally, note that the error term is independent of $\sigma$ and hence even when $\sigma=0$, the error term is non-trivial.


\begin{corollary}\label{cor:gauss_ssgd_utility}
Consider the setting of Theorem~\ref{thm:sgd-utility}. Let $\vx_i \sim \cN(\vzero, \vI) \in \bR^d$, i.e., $R_x^2=d$ and $\kappa = 1$. Let $y_i= \langle \vx_i, \vw^* \rangle + \fieldn_i$ where $\fieldn_i \sim \cN(0,\sigma^2)$. Then, assuming $N\geq d^2 \log \frac{ \epsilon}{\log (1/\delta)}$, we have:  
\begin{enumerate}
    \item \dpssgd  with parameters $\eta=1/4d$, $\zeta = \sqrt{d} \sqrt{\|\wo\|^2 + \sigma^2}\log N$, $\alpha = \frac{\log(N/\delta)}{\varepsilon \sqrt{N}}$ is $(\epsilon, \delta)$-DP where $\epsilon=\cO\Big(\sqrt{\frac{\log(1/\delta)}{N}}\Big)$. 
    \item The output $\overline{\vw}$ satisfies the following risk bound w.p. $\geq 1-1/N^{100}$: $ \cL(\overline{\vw}) - \cL(\vw^*)
    = \widetilde{\cO}\Big( \frac{d\sigma^2}{N} + \frac{d^2 (\|\wo\|^2+\sigma^2)}{N^2\varepsilon^2}\log^2(N/\delta)\Big).$
\end{enumerate} 
\end{corollary}


\noindent Below, we provide a proof of Theorem~\ref{thm:sgd-utility}. See Appendix~\ref{app:sgd} for the supporting lemmas and their detailed proofs.

\begin{proof}[Theorem~\ref{thm:sgd-utility}]
Consider the event: 
$\cE=\{\|\vx_t(\langle \vx_t, \vw_t\rangle - y_t)\|\leq \zeta, \forall\; 0\leq t\leq N-1\}.$ 
Note that by Lemma~\ref{lemma:ssgd-gradient-norm-bound}, 
$    Pr(\cE)\geq 1-\frac{1}{N^{100}}.$ 
Now, if $\cE$ holds, then DP-SSGD does not have any clipping. 
Thus, under $\cE$, by Lemma~\ref{lemma:sgd-zeta-infinity-utility}: 
\begin{equation}
\resizebox{\textwidth}{!}{
$\cL(\overline{\vw}) - \cL(\vw^*) 
    \leq
    \frac{R_x^2}{2}\exp\Big(-\frac{\eta\mu}{2}N\Big) \|\wo\|^2+ \frac{1}{N}\text{Tr}(\vH^{-1}(\vSigma+4\zeta^2\alpha^2 \vI)) + \frac{1}{N}\cdot \frac{\eta R_x^2}{1-\eta R_x^2}d\|\vSigma+4\zeta^2\alpha^2 \vI\|_\vH,$
    }
\end{equation}
where $\vSigma=\E{\fieldn^2 \vx\vx^\s{T}}$. Note that $\|\vSigma\|_2\leq \sigma^2 \|\vH\|_2$. 
Using $\zeta$, $\alpha$, $\eta$ as defined in the theorem, we have (assuming event $\cE$): 
\begin{multline}
\cL(\overline{\vw}) - \cL(\vw^*) 
    \leq
    \exp\left(-\frac{\mu c}{2R_x^2}\cdot \frac{\epsilon^2 N}{K_2^2 \kappa d \log^2(N/\delta)\log^{4a+2} N}\cdot N\right) \frac{d}{2} \|\wo\|^2_{\vH}+ \frac{2}{N}\sigma^2 d\\+  \frac{128 d R_x^2 K_2^2\kappa\log^{4a}N ( \|\wo\|^2_{\vH}+ \sigma^2)\log^2 (N/\delta)}{\mu \epsilon^2 N^2}. 
\end{multline}
Theorem now follows by using the sample complexity bound on $N$, the fact that $R_x^2 \leq d \|\vH\|_2$, $\kappa=\|\vH\|_2/\mu$ and by combining the above inequality with $\Pr(\cE)\geq 1-1/N^{100}$. 
\end{proof}

\begin{lemma}\label{lemma:sgd-zeta-infinity-utility}
Let $D=\{(\vx_i, y_i)\}_{i=0}^{N-1}$ be sampled from distribution $\cD$; see Section~\ref{sec:prob} for the notations and assumptions about $\cD$. Let $\wo$ minimize  the population squared loss. Let $\overline{\vw}$ be the output of \dpssgd (Algorithm~\ref{alg:dp-sgd}) s.t. there is no gradient clipping at any step and $\eta\leq \frac{1}{2R_x^2}$. Then we have: 
\begin{equation*}
\resizebox{\textwidth}{!}{$
    \cL(\overline{\vw}) - \cL(\vw^*) \leq \frac{R_x^2}{2}e^{-\eta\mu (\frac{N}{2}+1)}\|\vw^*\|^2 + \frac{1}{N}\frac{\eta R_x^2}{1-\eta R_x^2} d\|\vSigma + 4\zeta^2\alpha^2\vI\|_\vH + \frac{1}{N}\s{Tr}(\vH^{-1}(\vSigma+4\zeta^2\alpha^2\vI)).$
    }
\end{equation*}
\end{lemma}
\noindent See Appendix~\ref{app:sgd} for a detailed proof of the above lemma and the other supporting lemmas. 


\subsection{DP-Adaptive Mini-batch Shuffled SGD}\label{sec:ambssgd}
Theorem~\ref{thm:sgd-privacy} indicates that \dpssgd only applies in the setting where $\epsilon$ is small which can be restrictive for practical usage. The key challenge is that shuffling amplification results along with privacy composition over $T$ iterations require per-iteration privacy budget to be less than $1/\sqrt{T}$. As $T=N$ for \dpssgd, it implies that $\epsilon=O(1/\sqrt{N})$.

So the key challenge is to be able to reduce the number of iterations despite   one-pass shuffling  so that we do not have additional $d$ factors. To this end, we process a batch of points in each update of shuffled SGD instead of using a single point per iteration as in \dpssgd. As mentioned above, another weakness in the result of \dpssgd is that the error bound does not vanish even for vanishing noise variance $\sigma\rightarrow 0$; recall that $\sigma$ is the  variance of noise in observations $\E{\fieldn^2}$. We fix this issue by adaptively estimating the clipping threshold for the gradients. Note that a similar technique was used by \cite{biswas2020coinpress} but in a  different context of mean estimation. 

See Algorithm~\ref{alg:ambssgd} for a pseudo code of our method. The algorithm randomly shuffles the data (Step 2) and takes one pass over it. The $N$ data points are divided into $T\approx \kappa \log N$ batches of size $b+s$, where  $\kappa$ is the condition number of $\vH=\E{\vx\vx^\s{T}}$, i.e., $\kappa=\|\vH\|_2/\mu$. Each iteration first estimates the clipping threshold $\zeta_t$ while preserving differentially privacy. Ideally we require $\zeta_t$ to be large enough so that it does not lead to any thresholding (with high probability). The gradient of a given point $(\vx, y)$ is $\vx (\langle \vx, \vw_t\rangle -y)$. Furthermore, by using the sub-Gaussian style assumptions on $\cD$ (see \eqref{eq:assump_vx}, \eqref{eq:xnorm}) we have $(\langle \vx, \vw_t\rangle -y)^2 \lessapprox \E{ (\langle \vx, \vw_t\rangle -y)^2}\cdot \log^a N$ if $\vx, y$ are independent of $\vw_t$. If $\vw_t\approx \wo$, then the above quantity decreases to around $\sigma^2$ which is desired to get stronger bounds. 

To estimate $\zeta_t$, we use \dpstat method (see Algorithm~\ref{alg:dpstat}) which employs a standard technique from the private statistics literature to approximately  estimate the highest value in the dataset. At a high level, \dpstat  ( Algorithm~\ref{alg:dpstat}) exploits the fact that a single point cannot significantly perturb say $99^\s{th}$ percentile of a given batch of points, and adds noise appropriately. 

After estimating $\zeta_t$ (Step 7, Algorithm~\ref{alg:ambssgd}), we select the next batch of points and apply the standard mini-batch SGD update but with  added Gaussian noise of standard deviation $2\eta\zeta_t\alpha/b$ where $\alpha$ is the standard noise multiplier to ensure differential privacy (Step 10). Finally, we output $\overline{\vw}$ which is the average of the last $T/2$ iterates (Step 12).


\begin{algorithm}[!t]
\caption{\dpambssgd:  DP-Adaptive-Mini-Batch-Shuffled-SGD}\label{alg:ambssgd}
\begin{algorithmic}[1]
    \STATE{\bfseries Input:} Samples: $\{(\vx_i, y_i)\}_{i=0}^{N-1}$, Learning Rate: $\eta$, DP Noise Multiplier: $\alpha$, Batch Size: $b$, Expected $\vx$ Norm: $R_x$, Domain Size: $B$, Stat Sample Size: $s$
    \STATE Randomly permute $\{(\vx_i, y_i)\}_{i=0}^{N-1}$
    \STATE $T=N/(b+s)$ 
    \STATE $\vw_{0} \gets \vzero$
    \FOR{$t=0 \dots T-1$}
        \STATE $\tau\leftarrow  (b+s)t$
        \STATE $\gamma_{t}\leftarrow $\dpstat$(\{(\vx_i, y_i)\}_{i=\tau}^{\tau+s}, \vw_{t}, B, s, \alpha, B/N^{200})$
        \STATE $\zeta_{t}\leftarrow R_x\cdot \gamma_{t}\cdot \log^{a} N$
        \STATE \text{Sample }$\vg_{t} \sim \cN(\vzero, \vI_{d\times d})$
        {\STATE $\vw_{t+1} \gets \vw_{t} - \eta\left(\frac{1}{b} \sum_{i=0}^{b-1}\text{ clip}_{\zeta_{t}}(\vx_{\tau+s+i}(\langle \vx_{\tau+s+i}, \vw_{t}\rangle - y_{\tau+s+i}))+\alpha\cdot\frac{2\zeta_t}{b}\cdot \vg_{t}\right)$ where $\text{clip}_\zeta(\vnu) = \vnu\cdot\min\Big\{1, \frac{\zeta}{\|\vnu\|_2}\Big\}$}
    \ENDFOR
\STATE{\bfseries Result:} $\overline{\vw} := \frac{2}{T}\sum_{t'=T/2+1}^{T}\vw_{t'}$
\end{algorithmic}
\end{algorithm}

\begin{algorithm}[!t]
\caption{\dpstat:  Private Estimation of Approximately Maximum}\label{alg:dpstat}
\begin{algorithmic}[1]
    \STATE{\bfseries Input:} Samples: $\{(\vx_i, y_i)\}_{i=0}^{s}$, Parameter: $\vw$, Domain Size: $B$, Stat Sample Size: $s$, Noise Multiplier: $\alpha$, Discretization Width: $\Delta$
    \STATE $\gamma_0 \leftarrow \Delta$.
    \FOR{$i\in\left\{0,\ldots,\lceil\log_2(B/\Delta)\rceil\right\}$}
        {\STATE $c\leftarrow\left|\left\{|\ip{\vx_j}{\vw}-y_j|\leq \gamma_i:j\in\{0,\ldots,s\}\right\}\right|$. $\privc\leftarrow c+\calN(0,\lceil\log_2(B/\Delta)\rceil\alpha^2)$.\label{line:pop}}
        \STATE {\bf if} $\privc<s$ {\bf then} $\gamma_{i+1}\leftarrow 2\cdot\gamma_i$, {\bf else break}. 
    \ENDFOR
    \STATE {\bf return} $\gamma_{\sf priv}\leftarrow\gamma_i$.
\end{algorithmic}
\end{algorithm}


\begin{theorem}\label{thm:ambsgd-privacy}
 Algorithm ${\sf DP-AMBSSGD}$ with noise multiplier $\alpha$ satisfies $\frac{1}{\alpha^2}$-zCDP, and correspondingly satisfies $(\epsilon,\delta$)-differential privacy when we set the noise multiplier $\alpha \geq \frac{2\sqrt{\log(1/\delta)+\epsilon}}{\epsilon}$. Furthermore, if $\epsilon\leq \log(1/\delta)$, then $\alpha \geq \frac{\sqrt{8\log(1/\delta)}}{\epsilon}$ suffices to ensure $(\epsilon,\delta)$-differential privacy.
\end{theorem}
{\bf Remarks:} The privacy proof follows by combining privacy analysis of clipped SGD with that of approximate maximum estimation (\dpstat). Since each data-sample is observed in only one mini-batch and the algorithm \dpambssgd takes only a single pass over the entire data (and by implication only a single pass each data point), we use parallel composition property of DP to avoid the $\sqrt{T}$ factor which would have manifested when employing sequential composition. Further, our algorithm and result holds for any $\epsilon$, in contrast with the  $\epsilon=O(1/\sqrt{N})$ requirement of \dpssgd.

\noindent We now present the excess risk bound for \dpambssgd. 
\begin{theorem}\label{thm:ambsgd-utility}
 Let $D=\{(\vx_i, y_i)\}_{i=0}^{N-1}$ be sampled i.i.d. with $\vx_i \sim \cD$ satisfying $(\vH,K_2,a,\hp)$-Tail, and the distribution of the inherent noise $\fieldn_i$ satisfies $(\sigma^2,K_2,a,\hp)$-Tail with $\beta=\frac{1}{N^{200}}$. Let $\wo$ be the optimal solution to the population least squares problem, $\kappa$ be the condition number of the covariance matrix $\vH$, and $\mathbb{E}\left[\ltwo{\vx}^2\right]\leq R_x^2$ (see Section~\ref{sec:prob} for  notations and assumptions).\\
Initialize parameters in \dpambssgd as follows: batch size $b =\frac{N}{T}-s$, stat sample size $s=\frac{b}{10}$,  stepsize $\eta = \frac{b}{R_x^2 + (b-1)\|\vH\|_2}$, number of iterations  $T = c_1\kappa \log(N)$ where $c_1>0$ is a global constant, domain size $B= K_2R_x (\|\wo\|_{\vH}+\sigma)\log ^{2a} N$, and noise multiplier $\alpha=\frac{\sqrt{8\log(1/\delta)}}{\epsilon}$. Then, the output $\overline{\vw}$  achieves the following excess risk w.p. $\geq 1-1/N^{100}$ over the randomness in data and algorithm: 
\begin{equation*}
    \cL(\overline{\vw}) - \cL(\vw^*) \leq 
    \frac{\|\wo\|_{\vH}^2}{N^{100}}+\frac{8\sigma^2 d}{N}
     + \softO{\frac{\sigma^2 \kappa^2d^2\log(1/\delta)}{\epsilon^2 N^2}},
\end{equation*}
assuming $N\geq \softOm{\kappa^2 d \left(1+\frac{\sqrt{\log(1/\delta)}}{\epsilon}\right)} $. 

\end{theorem}
{\bf Remarks:}
{\bf i)} Note that the excess risk bound has three terms. The first term depends on $\|\wo\|_{\vH}$ but is polynomially small in $N$. The second term matches the information theoretically optimum rate (up to constant factor) for non-private linear regression. Finally, the third term is approximately $\sigma^2 d^2/(\epsilon N)^2$.\\
{\bf ii)} Up to the first term which is $1/N^{100}$ and up to $\kappa$, $\log N$ factors, the bounds match the lower bound of $\sigma^2 d/N + \sigma^2 d^2/(\epsilon^2 N^2)$ by \cite{Cai21}. Note that the first term is present as \dpambssgd can take only pass over the data. However, by using DP-SGD with initialization $\vw_0\leftarrow \overline{\vw}$, and running the algorithm for $\log (N/\sigma)$ iterations, we can get rid of the first term. One key observation here is that, as the error term dependent on $\|\wo\|_{\vH}$ is already small, the {\em additional} $d$ term introduced by DP-SGD is inconsequential to the overall bound. As the DP-SGD part is relatively standard, we do not present the above mentioned argument in this paper.\\
{\bf iii)} The result holds for nearly optimal sample complexity of $N=\softOm{d}$, ignoring $\kappa$ and $\log N$ factors. This matches the sample complexity bound for {\em non-private} linear regression as well. Furthermore, time complexity of our method is $O(Nd)$ which is linear in the input size.\\
{\bf iv)} The leading term in the error bound corresponding to the MLE error of $\sigma^2 d/N$ is independent of $\kappa$ and $\log N$ factors. But the last term as well as sample complexity depends on $\kappa$. The exact dependence can be optimized by more careful analysis. Furthermore, using a slightly more complicated algorithm that forms multiple batch sizes and tail averaging in each phase can perhaps further reduce dependence on $\kappa$. However, it is not clear if the dependence can be removed completely; we leave further investigation into dependence on $\kappa$ for future work. \\
{\bf v)} Consider the standard setting of linear regression with $\vx\sim \cN(\vzero, \vI)$ and $\fieldn \sim \sigma \cN(0,1)$. Here, our bound is given by the below corollary which is nearly optimal up to the first term and $\log N$ factors. Our second term is smaller then existing results for the same setting by $O(d)$ and a $\sigma^2$ factor. \\
{\bf vi)} We want to clarify that the DP guarantees should hold over worst-case data sets whereas the high-probability that gradient clipping
does not happen is over the randomness of the data. In our privacy analysis, one cannot exploit this randomness and get rid of the clipping. However, if clipping indeed happens, one can exactly quantify the clipped loss as a Huberized version of the least squared error loss on individual data points. (For a detailed discussion, See \cite[Section 5.1]{song2020characterizing})
\begin{corollary}
Consider the setting of Theorem~\ref{thm:ambsgd-utility}. Let $\vx_i \sim \cN(\vzero, \vI) \in \bR^d$, i.e., $R_x^2=d$ and $\kappa = 1$. Let $y_i= \langle \vx_i, \vw^* \rangle + \fieldn_i$ where $\fieldn_i \sim \cN(0,\sigma^2)$. Then, assuming $N\geq \softOm{d(1+\frac{\sqrt{\log(1/\delta)}}{\epsilon})}$, we have:  
\begin{enumerate}
    \item Algorithm \dpambssgd  with parameters $\eta=1/4d$,  $\alpha = \frac{\sqrt{8\log(1/\delta)}}{\varepsilon}$ is $(\epsilon, \delta)$-DP. 
    \item The output $\overline{\vw}$ satisfies the following risk bound: 
\begin{equation*}
    \cL(\overline{\vw}) - \cL(\vw^*)\leq   \frac{\|\wo\|_{\vH}^2}{N^{100}}+\frac{8 d\sigma^2}{N} \left( 1+ \softO{\frac{d \log(1/\delta)}{N\varepsilon^2}}\right).
\end{equation*}
\end{enumerate} 
\end{corollary}

\begin{proof}[Theorem~\ref{thm:ambsgd-utility}] 
Let $\Gamma=\Omega\left(\frac{\sqrt{\log (1/\delta)}}{\epsilon}\cdot \log N\right)$. Then, using Lemma~\ref{lem:privadpstat}, with probability  $\geq 1-1/N^{200}$, at least $s-\Gamma$ points satisfies $|\langle \vx_i, \vw_t\rangle -y_i|\leq \gamma_t$. Hence, using $s=b/10$ and sample complexity assumption, we have: $\kappa\|\wo-\vw_t\|_{\vH}^2 + \sigma^2 \leq \gamma_t^2$ with probability $\geq 1-1/N^{200}$. Hence, by setting $\zeta_t=R_x \gamma_t \log^{a} N$ and using assumptions \eqref{eq:xnorm}, \eqref{eq:assump_vx}, along with the fact that $\vw_t$ is independent of $\vx_{\tau+s+i}$ for $0\leq i\leq b-1$, we have (w.p. $\geq 1- 1/N^{200}$): $$\|\vx_{\tau+s+i}(\langle \vx_{\tau+s+i}, \vw_t \rangle - y_{\tau+s+i})\|\leq \zeta_t,$$
for all iterations $t$. That is, with probability $\geq 1- 1/N^{199}$, the following event holds: $\cE:=\{\text{thresholding is not required for any point}\}$. 
Now, using the second part of the utility lemma, $s=b/10$, and the sample complexity as above, we also know that $\gamma_{t} \leq K_2\log^{a} N\Big(\sqrt{\kappa}\|\vw_{t}-\wo\|_{\vH}+ \sigma + \Delta\Big)$ and $\Delta = \frac{\|\vw^*\|_{\vH} + \sigma}{N^{100}}$. Hence, the requirement for $\zeta_t$ given in Lemma~\ref{lemma:ambssgd-zeta-infinity-utility} is satisfied with probability $\geq 1- 1/N^{199}$. 

That is, with probability $\geq 1- 2/N^{199}$, both the above mentioned events hold, i.e., thresholding is not required for any point, and $\zeta_t$ requirement is satisfied for each iteration. Theorem now follows by applying  Lemma~\ref{lemma:ambssgd-zeta-infinity-utility} conditioned over the above two high probability events. The final expression in the theorem is obtained by applying sample complexity bound assumption along with fact that $R_x^2 \leq d \|\vH\|_2$ and $\kappa=\|\vH\|_2/\mu$. 

\end{proof}
\begin{lemma}[\dpstat]\label{lemma:dp-stat}
 In the following we provide the privacy and utility guarantees:\\
 ({\bf Privacy}) Algorithm \dpstat  satisfies $\frac{1}{2\alpha^2}$-zCDP.\\
 ({\bf Utility}) Let $\Gamma=\alpha\sqrt{2\log (B/\Delta)\log(\log(B/\Delta)/\beta)}$. Now, w.p. at least $1-\beta$, Algorithm \dpstat outputs $\gamma_{\sf priv}$ s.t.   $\left|\left\{|\ip{\vx_i}{\vw}-y_i|\leq \gamma_{\sf priv}:i\in\{0,\ldots,s\}\right\}\right|\geq s-\Gamma,$ and\\ $\left|\left\{|\ip{\vx_i}{\vw}-y_i|\leq \max\left\{\frac{\gamma_{\sf priv}}{2},\Delta\right\}:i\in\{0,\ldots,s\}\right\}\right|<s-\Gamma.$ 
\label{lem:privadpstat}
\end{lemma}

\begin{lemma}\label{lemma:ambssgd-zeta-infinity-utility}
Let $D=\{(\vx_i, y_i)\}_{i=0}^{N-1}$ be sampled from distribution $\cD$; see Section~\ref{sec:prob} for the notations and assumptions about $\cD$. Let $\wo$ be the optimal solution to the population least squares problem. Let $\overline{\vw}$ be the output of \dpambssgd (Algorithm~\ref{alg:ambssgd}) s.t. there is no gradient clipping at any step, $\eta\leq \frac{b}{R_x^2 + (b-1)\|\vH\|}$, batch size $b$, total number of iterations $T$ s.t. $T\cdot (b+s) = N$ and $T = c_1\kappa\log N$ where $c_1>0$ is a global constant and $\Big(\frac{N}{T}-s\Big)^2 \geq \frac{24\eta \alpha^2R_x^2 K_2^2\kappa\log^{4a}N\s{Tr}(\vH)}{\mu}$. Then we have:   
\begin{align*}
&\cL(\overline{\vw}) - \cL(\vw^*) 
    \leq \Big(\frac{4}{\eta^2\mu^2T^2}e^{-\eta\mu} + \frac{96\alpha^2}{Tb^2}K_2^2R_x^2\kappa\log^{4a} N\s{Tr}(\vH^{-1})\Big)\cdot\\
    &\qquad\qquad\qquad\qquad\Big(\frac{\|\vw^*\|^2_\vH}{N^{\frac{\eta\mu}{4}c_1\kappa}} + \frac{2\eta}{\mu b}\s{Tr}(\vH\vSigma) + \frac{24\eta\alpha^2}{\mu b^2} K_2^2R_x^2\log^{4a} N(\sigma^2 + \Delta^2) \s{Tr}(\vH)\Big) \\
    &\qquad\qquad\qquad\quad + \frac{8}{Tb}\Big(\s{Tr}(\vH^{-1}\vSigma) + \frac{12\alpha^2}{b}K_2^2R_x^2\log^{4a}N(\sigma^2+\Delta^2)\s{Tr}(\vH^{-1})\Big),
\end{align*}
where $\zeta_{t} \leq R_xK_2\log^{2a} N\Big(\sqrt{\kappa}\|\vw_{t}-\wo\|_{\vH}+ \sigma + \Delta\Big)$ and $\Delta = \frac{\|\vw^*\|_{\vH} + \sigma}{N^{100}}$.
\end{lemma}
See Appendix~\ref{app:ambsgd} for a detailed proof of the above lemma and the other supporting lemmas. 


\section{Conclusions}
We studied the problem of differentially private linear regression, and proposed a novel algorithm \dpambssgd based on mini-batch shuffled SGD with adaptive clipping. The method can guarantee nearly optimal error rate in terms of $d$, $N$, $\sigma$, and $\epsilon$, for sub-Gaussian style distributions, a significant improvement over existing risk bounds for polynomial time methods. However, our results depend critically on the condition number $\kappa$ of Hessian $\E{\vx\vx^\s{T}}$. An investigation of both the lower bound as well as the upper bound w.r.t. $\kappa$ is an important future direction. Furthermore, removing $\log$ dependencies in the analysis is interesting. Finally, generalizing our techniques to obtain optimal error rates for generalized linear models like logistic regression should be of wide interest.

\clearpage
\bibliography{references}

\newpage
\appendix
\section{Missing Proofs from Section 3.1}\label{app:sgd}
\subsection{Proof of Privacy Guarantees}
\subsubsection{Auxiliary Lemmas}
Here we present a few auxiliary results which will be used in proving our \dpssgd guarantees.
\begin{lemma}\label{dpssgd-privacy-lemma-1}
Each update step of \dpssgd (Algorithm~\ref{alg:dp-sgd}) is $(\epsilon_0, \delta_0)$ private if $\alpha = \frac{ c}{\epsilon_0}$ where $c \geq \sqrt{2\log(1.25/\delta_0)}$ and $L$ is the clipping norm.
\end{lemma}

\begin{proof} Each update step (excluding the DP-noise addition) is of the form:
\begin{align*}
    \vw_{t+1} \leftarrow \vw_{t} - \eta\text{ clip}_\zeta(\vx_t(\langle \vx_t, \vw_t\rangle - y_t)),
\end{align*}
where $\text{clip}_\zeta(\vnu) = \vnu\cdot\max\Big\{1, \frac{\zeta}{\|\vnu\|_2}\Big\}$.

Therefore, the local $L_2$ sensitivity of the $\vw_{t+1}$ can be computed by considering a difference in the $t^\s{th}$ iteration data sample as follows:
\begin{align*}
    \Delta_2 &= \|\vw'_{t+1} - \vw_{t+1}\|\\
    &= \|\eta\text{ clip}_\zeta(\vx_t'(\langle \vx_t', \vw_t\rangle - y_t')) - \eta\text{ clip}_\zeta(\vx_t(\langle \vx_t, \vw_t\rangle - y_t))\|\\
    &\leq 2\eta\|\text{ clip}_\zeta(\vx_t(\langle \vx_t, \vw_t\rangle - y_t))\|\\
    &= 2\eta \zeta.
\end{align*}

Using Gaussian mechanism to induce Local Differential Privacy with the above $L_2$ sensitivity proves the Lemma.
\end{proof}

\begin{theorem}[From \cite{feldman2021hiding}]\label{thm:shuffling-privacy}
 For a domain $\cD$, let $\cR^{(i)}: f \times \cD \rightarrow \cS^{(i)}$ for $i \in [n]$ (where $\cS^{(i)}$ is the range space of $\cR^{(i)}$) be a sequence of algorithms such that $\cR^{(i)}(z_{1:i-1}, \cdot)$ is a $(\epsilon_0, \delta_0)$-DP local randomizer for all values of auxiliary inputs $z_{1:i-1} \in \cS^{(1)} \times \cdots \times \cS^{(i-1)}$. Let $\cA_s: \cD^n \rightarrow \cS^{(1)} \times \cdots \times \cS^{(n)}$ be the algorithm that given a dataset $x_{1:n} \in \cD^n$, samples a uniformly random permutation $\pi$, then sequentially computes $z_i = \cR^{(i)}(z_{1:i-1}, x_{\pi(i)})$ for $i \in [n]$, and outputs $z_{1:n}$. Then for any $\delta \in [0,1]$ such that $\epsilon_0 \leq \log \Big(\frac{n}{16\log (2/\delta)}\Big)$, $\cA_s$ is $(\epsilon, \delta + \cO(e^{\epsilon}\delta_0n))$-DP where $\epsilon$ is:
\begin{align*}
    \epsilon = \cO\Big((1-e^{-\epsilon_0})\Big(\frac{\sqrt{e^{\epsilon_0}\log(1/\delta)}}{\sqrt{n}} + \frac{e^{\epsilon_0}}{n}\Big)\Big).
\end{align*}
\end{theorem}

\subsubsection{Proof of Theorem~\ref{thm:sgd-privacy}}

\begin{proof}
We will use Theorem \ref{thm:shuffling-privacy} to obtain the stated result. Denoting auxiliary inputs using $u$, we can rewrite the update steps of \dpssgd as a sequence of one step algorithms:
\begin{align*}
    \cR^{(t+1)}(u_{0:t}, (\vx, y)) := \vw_{t+1} \gets \vw_{t}(u_{0:t}) - \eta\text{ clip}_\zeta(\vx_{\pi(t)}(\langle \vx_{\pi(t)}, \vw_{t}(u_{0:t})\rangle - y_{\pi(t)})) - 2\eta\zeta\alpha g_{t},
\end{align*}
where $\pi(t)$ denotes the $t^\s{th}$ iteration sample after randomly permuting the input data. 

From Lemma \ref{dpssgd-privacy-lemma-1}, each $\cR^{(t+1)}(u_{0:t}, \cdot)$ is a $(\epsilon_0, \delta_0)$-DP local randomizer algorithm with $\epsilon_0 \leq \log\Big(\frac{N}{16\log(2/\widehat{\delta})}\Big)$. The output of \dpssgd is obtained by post-processing of the shuffled outputs $u_{t+1} = \cR^{(t+1)}(u_{0:t}, (\vx, y))$ for $t \in \{0, \dots, N-1\}$.

Therefore, Theorem \ref{thm:shuffling-privacy} implies that Algorithm \dpssgd  is $(\widehat{\epsilon}, \widehat{\delta} + \cO(e^{\widehat{\epsilon}}\delta_0N))$-DP such that:
\begin{align*}
    \widehat{\epsilon} &= \cO\Big((1-e^{-\epsilon_0})\Big(\frac{\sqrt{e^{\epsilon_0}\log(1/\widehat{\delta})}}{\sqrt{N}} + \frac{e^{\epsilon_0}}{N}\Big)\Big)
\end{align*}

Now, assume that $\epsilon_0 \leq \frac{1}{2}$. The above implies $\exists\; c_1 > 0$ s.t.
\begin{align*}
    \widehat{\epsilon} &\leq c_1 \cdot (1-e^{-\epsilon_0})\Big(\frac{\sqrt{e^{\epsilon_0}\log(1/\widehat{\delta})}}{\sqrt{N}} + \frac{e^{\epsilon_0}}{N}\Big)\\
    &\leq c_1 \cdot \Big( (e^{\epsilon_0/2}-e^{-\epsilon_0/2})\sqrt{\frac{\log(1/\widehat{\delta})}{N}} + (e^{\epsilon_0}-1)\frac{1}{N}\Big)\\
    &\leq c_1\cdot \Big(((1+\epsilon_0) - (1-\epsilon_0/2))\sqrt{\frac{\log(1/\widehat{\delta})}{N}} + ((1+2\epsilon_0) - 1)\frac{1}{N}\Big)\\
    &= c_1 \cdot \epsilon_0 \Big(\frac{1}{2}\sqrt{\frac{\log(1/\widehat{\delta})}{N}} + \frac{2}{N}\Big) \numberthis\label{eq:ssgd-privacy-lemma-epsilon-epsilon0}.
\end{align*}
From Lemma \ref{dpssgd-privacy-lemma-1}, we set $\alpha=\frac{\sqrt{2\log(1.25/\delta_0)}}{\epsilon_0}$, then each update step of \dpssgd independently satisfies $(\epsilon_0,\delta_0)$-DP. This follows from standard Gaussian mechanism arguments~\citep{mironov2017renyi,ODO}. In~\eqref{eq:ssgd-privacy-lemma-epsilon-epsilon0} we replace $\epsilon_0=\frac{\sqrt{2\log(1.25/\delta_0)}}{\alpha}$ to obtain the following:
\begin{equation}
    \widehat{\epsilon}\leq c_1\cdot\frac{\sqrt{2\log(1.25/\delta_0)}}{\alpha}\cdot\Big(\frac{1}{2}\sqrt{\frac{\log(1/\widehat{\delta})}{N}} + \frac{2}{N}\Big)\leq c_1\cdot\frac{\sqrt{2\log(1.25/\delta_0)\log(1/\widehat\delta)}}{\alpha\sqrt N}.\label{eq:opop1}
\end{equation}
Now to ensure that we satisfy overall $(\epsilon,\delta)$-DP, we will set $\widehat\delta=\frac{\delta}{2}$, and $\delta_0=c_2\cdot\frac{\delta}{e^{\widehat\epsilon}N}$ for some universal constant $c_2>0$. From~\eqref{eq:opop1} we have the following:
\begin{equation}
    \widehat\epsilon\leq c_1\cdot\frac{\sqrt{2\log(c_2\cdot 1.25\cdot e^{\widehat\epsilon}N/\delta)\cdot\log(2/\delta)}}{\alpha\sqrt{N}}\label{eq:opo290}
\end{equation}
For any $\epsilon\leq 1$, if we set $\alpha= c_3\cdot\frac{\log(N/\delta)}{\epsilon\sqrt N}\geq c_3\frac{\sqrt{\log(N/\delta)\log(1/\delta)}}{\epsilon\sqrt{N}}$ for sufficiently large $c_3>0$, then from~\eqref{eq:opo290} we have $\widehat\epsilon\leq \epsilon$. Furthermore, we need $\epsilon_0=\frac{\sqrt{2\log(1.25/\delta_0)}}{\alpha}<\frac{1}{2}$ due to the assumption in Theorem~\ref{thm:shuffling-privacy}, which will be ensured if we set $\epsilon=\cO\left(\sqrt\frac{{\log(N/\delta)}}{N}\right)$.

This immediately implies for $\alpha=\Omega\left(\frac{\log(N/\delta)}{\epsilon\sqrt N}\right)$, \dpssgd satisfies $(\epsilon,\delta)$-DP as long as $\epsilon=\cO\left(\sqrt\frac{\log(N/\delta)}{N}\right)$, which completes the proof.

\end{proof}

\subsection{Proof of Risk Bounds}
\subsubsection{Auxiliary Technical Lemmas}
Here we present a few auxiliary results which will be used in proving our \dpssgd risk bounds.

\begin{lemma}[From \cite{rudelson2013hansonwright}]\label{lemma:hanson-wright}
Hanson-Wright Inequality: For any $\vX \sim \cN(\vzero,\vSigma)$, the following holds for $t \geq 0$:
\begin{align*}
    \bP(\|\vX\|^2 \geq \s{Tr}(\vSigma) + 2\sqrt{t}\|\vSigma\|_\s{F} + 2t\|\vSigma\|_\s{op}) \leq e^{-t}.
\end{align*}
\end{lemma}

\begin{lemma}\label{lemma:ssgd-markov}
Let $\vB_t := \vI - \eta\vx_t\vx_t^\s{T}$ s.t. each $\vx_t \in \bR^{d} \,\, \forall\;n t \in \{j, \dots, T-1\}$ has been sampled i.i.d. from $\cD$. Let $\vz \in \bR^{d}$ be a vector independent of all $\vB_t$'s. Then for $\hp > 0$ and $\eta < \frac{1}{R_x^2}$, we have with probability $\geq 1-\hp$:
\begin{align*}
    \|\vB_{T-1}\vB_{T-2}\dots \vB_j\vz\|^2 \leq \frac{1}{\hp}e^{-\eta\mu(T-j)}\|\vz\|^2.
\end{align*}
\end{lemma}

\begin{proof}
Note that $\Ee{(\vx, y) \sim \cD}{\|\vB_{T-1}\vB_{T-2}\dots \vB_{j}\vz\|^2}$
\begin{align*}
    &= \Ee{(\vx, y) \sim \cD}{(\vB_{T-1}\vB_{T-2}\dots \vB_{j}\vz)^\s{T}\vB_{T-1}\vB_{T-2}\dots \vB_{j}\vz}\\
    &= \Ee{(\vx, y) \sim \cD}{\vz^\s{T}\vB_{j}^\s{T}\dots \vB_{T-2}^\s{T}\vB_{T-1}^\s{T}\vB_{T-1}\vB_{T-2}\dots \vB_{j}\vz}\\
    &= \Ee{(\vx, y) \sim \cD}{\vz^\s{T}\vB_{j}^\s{T}\dots \vB_{T-2}^\s{T}\Ee{(\vx_{T-1}, y_{T-1}) \sim \cD}{\vB_{T-1}^\s{T}\vB_{T-1}}\vB_{T-2}\dots \vB_{j}\vz}\\
    &\leq \Ee{(\vx, y) \sim \cD}{\vz^\s{T}\vB_{j}^\s{T}\dots \vB_{T-2}^\s{T}(\vI - 2\eta\vH + \eta^2R_x^2\vH)\vB_{T-2}\dots \vB_{j}\vz}\\
    &\leq \Ee{(\vx, y) \sim \cD}{\lambda_{\max}(\vI - 2\eta\vH + \eta^2R_x^2\vH)\|\vB_{T-2}\dots \vB_{j}\vz\|^2}\\
    &= \lambda_{\max}(\vI - 2\eta\vH + \eta^2R_x^2\vH)\Ee{(\vx, y) \sim \cD}{\|\vB_{T-2}\dots \vB_{j}\vz\|^2}\\
    &\leq (1 - \eta (2 - \eta R_x^2) \mu)\Ee{(\vx, y) \sim \cD}{\|\vB_{T-2}\dots \vB_{j}\vz\|^2} \\
    &\leq (1 - \eta \mu)\Ee{(\vx, y) \sim \cD}{\|\vB_{T-2}\dots \vB_{j}\vz\|^2} \numberthis\label{eq:eta_approx}\\ 
    &\leq \dots\\
    &\leq (1 - \eta \mu)^{(T-j)}\|\vz\|^2\\
    &\leq e^{-\eta\mu(T-j)}\|\vz\|^2,
\end{align*}
where in \eqref{eq:eta_approx}, we have used the fact that $\eta < \frac{1}{R_x^2}$.

Using Markov Inequality, for $\hp > 0$:
\begin{align*}
    \text{Pr}\Big\{\vZ \geq \frac{\E{\vZ}}{\hp}\Big\} &\leq \hp\\
    \implies \text{Pr}\Big\{\|\vB_{T-1}\dots \vB_j\vz\|^2 \leq  \frac{\Ee{(\vx, y) \sim \cD}{\|\vB_{T-1}\dots \vB_j\vz\|^2}}{\hp}\Big\} &\geq 1 - \hp.
\end{align*}
Therefore with probability at least $1-\hp$:
\begin{align*}
    \|\vB_{T-1}\dots \vB_0\vz\|^2 \leq \frac{1}{\hp}e^{-\eta\mu (T-j)}\|\vz\|^2.
\end{align*}
\end{proof}

\begin{lemma}[From \cite{sgd}]\label{lemma:ssgd-variance}
If the SGD update step is given by:
\begin{align*}
    \vw_{t+1} &= \vw_{t} - \eta(\vx_t(\langle \vx_t, \vw_t\rangle - y_t)),
\end{align*}
then,
\begin{align*}
    \frac{1}{2}\Ee{(\vx, y) \sim \cD}{\|\overline{\vw}_{t:T} - \vw^*\|^2_{\vH} | \vw_0 = \vw^*} \leq \frac{\s{Tr}(\vC_\infty)}{\eta (T-t+1)},
\end{align*}
where
\begin{align*}
    \overline{\vw}_{t:T} &:= \frac{1}{T-t+1}\sum_{t'=t}^{T}\vw_{t'},\\
    \s{Tr}(\vC_\infty) &\leq \frac{1}{2}\frac{\eta^2R_x^2}{1-\eta R_x^2}d\|\vSigma\|_\vH + \frac{\eta}{2}\s{Tr}(\vH^{-1}\vSigma).
\end{align*}
\end{lemma}

\begin{lemma}\label{lemma:ssgd-gradient-norm-bound}

Let $\eta$ be such that $\eta \leq \min\Big\{\frac{c_1}{\log^{4a+2} N}\cdot  \frac{1}{K_2^2R_x^2 \kappa}\cdot \frac{1}{d\alpha_{\epsilon, \delta, N}^2}, \frac{c_2}{\log^{2a+2} N}\cdot \frac{1}{R_x^2}\Big\}$, where $c_1, c_2>0$ are global constants and $\zeta = 4 K_2R_x\cdot \log^{2a}N\cdot \Big(\sqrt{\|\vH\|_2}\|\wo\| + \sqrt{\kappa}\sigma \Big)$. Furthermore, let $\alpha = \alpha_{\epsilon, \delta, N}$ be a function of $\epsilon, \delta, N$. Then, with probability $\geq 1 - \frac{1}{N^{100}}$, $\|\vx_t(\langle \vx_t, \vw_t\rangle - y_t)\|\leq \zeta$ for all $0\leq t\leq N-1$; $\vw_t$ is the $t^\s{th}$ iterate of Algorithm \dpssgd. 
\end{lemma}


\begin{proof}
We will prove the Lemma using Principle of Complete Induction. \\
\textbf{Base Case:} Check for $t = 0$.\\
The norm of the gradient is given by:
\begin{align*}
    \|\vx_0(\langle \vx_0, \vw_0\rangle - y_0)\| &= \|\vx_0(\langle \vx_0, \vzero \rangle - y_0)\|\\
    &= \|\vx_0y_0\|\\
    &\leq \|\vx_0\||y_0|\\
    &\leq R_x\log^{a}(1/\hp_x)\Big(K_2\sqrt{\|\vH\|_2}\|\vw^*\|\log^{a}(1/\hp_{\vw^*}) + \sigma K_2\log^{a}(1/\hp_\sigma)\Big)\\
    &= K_2R_x\log^{a}(1/\hp_x)\Big(\sqrt{\|\vH\|_2}\|\vw^*\|\log^{a}(1/\hp_{\vw^*}) + \sigma \log^{a}(1/\hp_\sigma)\Big).
\end{align*}
$\because$ from Section \ref{sec:prob}, we have $\|\vx_0\| \leq R_x\log^{a}(1/\hp_x)$ w.p. at least $1-\hp_x$ and $|y_0| = |\langle \vx_0, \vw^* \rangle + \fieldn_0| \leq |\langle \vx_0, \vw^* \rangle| + |\fieldn_0| \leq K_2\sqrt{\|\vH\|_2}\|\vw^*\|\log^{a}(1/\hp_{\vw^*}) + \sigma K_2\log^{a}(1/\hp_\sigma)$ w.p. at least $1-\hp_{\vw^*}-\hp_\sigma$. \\
As each $\hp$ is $1/{\sf poly(N)}$, the Lemma holds.\\\\
\textbf{Induction hypothesis:} The Lemma holds for iterations $t = 0, \dots, T-1$.\\\\
\textbf{Inductive case:} Check at $t = T$.\\
The norm of the gradient at iteration $t=T$ is $\|\vx_T(\langle \vx_T, \vw_T\rangle - y_T)\|$
\begin{align*}
    &= \|\vx_T\vx_T^\s{T}(\vw_T - \vw^*) + \vx_T\fieldn_T\|\\
    &\leq \|\vx\|(\|\vx_T^\s{T}(\vw_T - \vw^*)\| + \|\fieldn_T\|)\\
    &\leq R_x\log^{a}(1/\hp_x)\Big(K_2\sqrt{\|\vH\|_2}\|\vw_T - \vw^*\|\log^{a}(1/\hp_{\vw_T}) + \sigma K_2\log^{a}(1/\hp_\sigma)\Big)\\
    &= K_2R_x\log^{2a}N\Big(\sqrt{\|\vH\|_2}\|\vw_T - \vw^*\| + \sigma\Big),\numberthis\label{eq:ssgd-grad-bound-unplugged}
\end{align*}
since each $\hp_x, \hp_{\vw_T}, \hp_\sigma$ is $1/{\sf poly}(N)$. Therefore, we need to find a bound on  $\|\vw_T - \vw^*\|$.   

As $\zeta$ has not been exceeded in iterations $t = 0, \dots, T-1$, we have the following decomposition at iteration $T-1$:
\begin{align*}
    \vw_{T} - \vw^* &= \vw_{T-1} - \vw^* - \eta\Big(\text{ clip}_\zeta(\vx_{T-1}(\langle \vx_{T-1}, \vw_{T-1}\rangle - y_{T-1})) + 2\zeta\alpha \vg_{T-1}\Big)\\
     &= \vw_{T-1} - \vw^* - \eta\Big(\vx_{T-1}(\langle \vx_{T-1}, \vw_{T-1}\rangle - y_{T-1}) + 2\zeta\alpha \vg_{T-1}\Big)\\
        &= (\vI - \eta\vx_{T-1}\vx_{T-1})(\vw_{T-1} - \vw^*) + \eta\vxi_{T-1} - 2\eta\zeta\alpha \vg_{T-1}\\
    &= \vB_{T-1}(\vw_{T-1} - \vw^*) + \eta(\vxi_{T-1} - 2\zeta\alpha \vg_{T-1})\\
    &= \vB_{T-1}(\vw_{T-1} - \vw^*) + \eta\vnu_{T-1}\\
    &= \vB_{T-1}\dots \vB_0(\vw_0 - \vw^*) + \eta(\vnu_{T-1} + \vB_{T-1}\vnu_{T-2} + \dots + \vB_{T-1}\dots \vB_1\vnu_0),\\
\end{align*}
where, 
\begin{equation*}
    \vB_t := \vI - \eta\vx_t\vx_t^\s{T}, \ \ \ 
    \vxi_t := \vx_t(y_t - \langle \vx_t, \vw^* \rangle) = \vx_t\fieldn_t,\ \ \ 
     \vnu_t := \vxi_t - 2\zeta\alpha\vg_t.
\end{equation*}
Therefore,
\begin{align*}
    &\|\vw_{T} - \vw^*\|^2 = \|\vB_{T-1}\dots \vB_0(\vw_0 - \vw^*) + \eta(\vnu_{T-1} + \vB_{T-1}\vnu_{T-2} + \dots + \vB_{T-1}\dots \vB_1\nu_0)\|^2\\
    &\resizebox{\textwidth}{!}{$\leq 2\Big(\|\underbrace{\vB_{T-1}\dots \vB_0(\vw_0 - \vw^*)}_{\text{Bias Term}}\|^2 + \|\underbrace{\eta(\vnu_{T-1} + \vB_{T-1}\vnu_{T-2} + \dots + \vB_{T-1}\dots \vB_1\nu_0)}_{\text{Variance Term}}\|^2\Big).$} \numberthis\label{eq:ssgd-bias-var-conc-split}
\end{align*}\\
\textbf{Bound on Bias Term:}\\
Since $\vw_0 - \vw^*$ is independent of all $\vx_i$'s, we have (w.p. $\geq 1 - T\hp_{\vB}$), 
\begin{align*}
    &\|\vB_{T-1}\dots \vB_0(\vw_0 - \vw^*)\|^2 \leq  \|\vB_{T-1}\dots \vB_0\|^2\|\vw_0 - \vw^*\|^2= \|\prod_{j=0}^{T-1} \vB_{j}\|^2\|\vw_0 - \vw^*\|^2\\&\qquad\qquad\stackrel{\zeta_1}{\leq}  \Big(\prod_{j=0}^{T-1} \|\vB_{j}\|^2\Big)\|\vw_0 - \vw^*\|^2\stackrel{\zeta_2}{\leq} \|\vw_0 - \vw^*\|^2\stackrel{\zeta_3}{=}  \|\vw^*\|^2, \numberthis\label{eq:ssgd-bias-conc}
\end{align*}
where $\zeta_1$ follows by $\|\vA\vB\|\leq \|\vA\|\|\vB\|$, $\zeta_2$ follows from using $\eta \leq \frac{1}{R_x^2\log^{2a}(1/\hp_{\vB})}$ and the fact that with probability $\geq 1 - \hp_{\vB}$, we have $\|\vx_j\|^2\leq 1/\eta$, i.e., 
\begin{align*}
    \vB_j = \vI - \eta \vx_j\vx_j^\s{T} \preceq \vI \implies \|\vB_j\| \leq \|\vI\| = 1.    
\end{align*}
Finally, $\zeta_3$ follows from $\vw_0 = \vzero$. \\\\
\textbf{Bound on Variance Term:}\\
$\|\eta(\vnu_{T-1} + \vB_{T-1}\vnu_{T-2}  + \dots + \vB_{T-1}\dots \vB_1\vnu_0)\|^2$
\begin{align*}
    &= \eta^2\|(\vxi_{T-1} - 2\zeta\alpha \vg_{T-1}) + \dots + \vB_{T-1}\dots \vB_1(\vxi_{0} - 2\zeta\alpha \vg_{0})\|^2\\
    &\leq 2\eta^2\Big(\|\vxi_{T-1} + \dots + \vB_{T-1}\dots \vB_1\vxi_{0}\|^2 + 4\zeta^2\alpha^2\|\vg_{T-1} +  \dots + \vB_{T-1}\dots \vB_1\vg_{0}\|^2\Big). \numberthis\label{var-conc-split}
\end{align*}
We will compute the bound on each of the two terms separately and add them up.\\\\
\textbf{Bound on $\|\vg_{T-1} +  \vB_{T-1}\vg_{T-2}  + \dots + \vB_{T-1}\dots \vB_1\vg_{0}\|$}:\\
As $\vg_{\tau}$ is a Gaussian vector for all $0 \leq \tau\leq T-1$, and $\vB_\tau$'s are all independent of $\vg_\tau$, we have: 
\begin{align*}
    \vg_{T-1} +  \vB_{T-1}\vg_{T-2}  + \dots + \vB_{T-1}\dots \vB_1\vg_{0}= \widetilde{\vg} \sim \cN(\vzero, \vQ),
\end{align*}
where the covariance $\vQ$ is given by:
\begin{align*}
    \vQ &= \vI + \vB_{T-1}\vB_{T-1}^\s{T} + \vB_{T-1}\vB_{T-2}\vB_{T-2}^\s{T}\vB_{T-1}^\s{T}  + \dots  + \vB_{T-1}\dots \vB_1\vB_1^\s{T}\dots \vB_{T-1}^\s{T}\\
    &= \vI + \sum_{\tau=1}^{T-1} (\prod_{j=\tau}^{T-1} \vB_{j})(\prod_{j=\tau}^{T-1} \vB_{j})^\s{T}\\
    &=\vI+\sum_{\tau=1}^{\widetilde{\tau}} (\prod_{j=\tau}^{T-1} \vB_{j})(\prod_{j=\tau}^{T-1} \vB_{j})^\s{T}+\sum_{\tau=\widetilde{\tau}+1}^{T-1} (\prod_{j=\tau}^{T-1} \vB_{j})(\prod_{j=\tau}^{T-1} \vB_{j})^\s{T}.
\end{align*}
Using Lemma~\ref{lemma:ssgd-markov},  we have w.p. $\geq 1-T \hp_{var1, max}$,
\begin{align*}
    \s{Tr}\left(\sum_{\tau=1}^{\widetilde{\tau}} (\prod_{j=\tau}^{T-1} \vB_{j})(\prod_{j=\tau}^{T-1} \vB_{j})^\s{T}\right) &= \sum_{\tau=1}^{\widetilde{\tau}} \s{Tr}\left((\prod_{j=\tau}^{T-1} \vB_{j})(\prod_{j=\tau}^{T-1} \vB_{j})^\s{T}\right) \\
    &\leq  \sum_{\tau=1}^{\widetilde{\tau}} \frac{d}{\hp_{var1, \min}}e^{-\eta \mu \cdot (T-\tau)}, \numberthis \label{eq:normlem1}    
\end{align*}
where $\hp_{var1, \min}$ and $\hp_{var1, \max}$ are the minimum and maximum $\hp$'s respectively across all the Lemma~\ref{lemma:ssgd-markov} invocations.

Furthermore, for any $\tau$, we have:
\begin{align*}
    \|(\prod_{j=\tau}^{T-1} \vB_{j})(\prod_{j=\tau}^{T-1} \vB_{j})^\s{T}\| &\leq \prod_{j=\tau}^{T-1} \| \vB_{j}\|^2 \leq 1 \numberthis\label{eq:normlem2}\\
\end{align*}
w.p. $\geq 1 - (T-\tau)\hp_{\vB}$, where we use the fact that since $\eta \leq \frac{1}{R_x^2\log^{2a}(1/\hp_{\vB})}$ w.p. $\geq 1 - \hp_{\vB}$,
\begin{align*}
    \vB_j = \vI - \eta \vx_j\vx_j^\s{T} \preceq \vI \implies \|\vB_j\| \leq \|\vI\| = 1.    
\end{align*}
Thus using \eqref{eq:normlem1} and \eqref{eq:normlem2}, we have: 
\begin{align*}
    \s{Tr}(\vQ) &= \s{Tr}(\vI) + \s{Tr}\left(\sum_{\tau=1}^{\widetilde{\tau}} (\prod_{j=\tau}^{T-1} \vB_{j})(\prod_{j=\tau}^{T-1} \vB_{j})^\s{T}\right) + \s{Tr}\left(\sum_{\tau=\widetilde{\tau}+1}^{T-1} (\prod_{j=\tau}^{T-1} \vB_{j})(\prod_{j=\tau}^{T-1} \vB_{j})^\s{T}\right)\\
    &\leq d + \sum_{\tau=1}^{\widetilde{\tau}} \frac{d}{\hp_{var1, \min}}e^{-\eta \mu \cdot (T-\tau)} + \sum_{\tau=\widetilde{\tau}+1}^{T-1} d\|(\prod_{j=\tau}^{T-1} \vB_{j})(\prod_{j=\tau}^{T-1} \vB_{j})^\s{T}\|\\
    &\leq d + \frac{d}{\hp_{var1, \min}} e^{-\eta \mu(T-\widetilde{\tau})} \frac{1}{1-e^{-\eta\mu}} + d(T-\widetilde{\tau}-1)\cdot 1\\
    &\leq d\cdot (T-\widetilde{\tau})+\frac{3d}{\hp_{var1, \min}} e^{-\eta \mu(T-\widetilde{\tau})}.
\end{align*}
Hence, selecting $T-\widetilde{\tau}= \frac{101\log(1/\hp_{var1, \min})}{\eta \mu}$, we have w.p. $\geq 1-T(\hp_{var1, \max} + \hp_{\vB})$,
\begin{equation}
    \s{Tr}(\vQ)\leq d\cdot \Big(\frac{101}{\eta \mu}\log (1/\hp_{var1, \min}) + 3\hp_{var1, \min}^{100}\Big).
\end{equation}
Now, using standard Gaussian property along with the bound on $\s{Tr}(\vQ)$ given above we have w.p. $\geq 1-T(\hp_{var1, \max} + \hp_{\vB}) -\hp_{tr}$,
\begin{align*}
\|\vg_{T-1} +  \vB_{T-1}\vg_{T-2}  + &\dots + \vB_{T-1}\dots \vB_1\vg_{0}\| \leq \sqrt{\s{Tr}(\vQ)} \sqrt{\log  (1/\hp_{tr})} \\
   &\leq \sqrt{d\log  (1/\hp_{tr})}\cdot \sqrt{\frac{101}{\eta \mu}\log (1/\hp_{var1, \min}) + 3\hp_{var1, \min}^{100}}. \numberthis\label{eq:ssgd-conc-var1-bound}
\end{align*}
\\\textbf{Bound on $\|\vxi_{T-1} +  \vB_{T-1}\vxi_{T-2}  + \dots + \vB_{T-1}\dots \vB_1\vxi_{0}\|$}\\
Note that
\begin{multline*}
\|\vxi_{T-1} +  \vB_{T-1}\vxi_{T-2}  + \dots + \vB_{T-1}\dots \vB_1\vxi_{0}\|\\= \|\vx_{T-1}\fieldn_{T-1} +  \vB_{T-1}\vx_{T-2}\fieldn_{T-2}  + \dots + \vB_{T-1}\dots \vB_1\vx_{0}\fieldn_{0}\|:= \|\vV\vfieldn\|,
\end{multline*}
where $\vfieldn$ is a vector having entries $[\fieldn_0, \dots, \fieldn_{T-1}]$ and $\vV$ is a $d\times T$ matrix with the $j^\s{th}$ column being the vector $\vB_{T-1}\dots \vB_{j}\vx_{j-1}$ and the $T^\s{th}$ column being the vector $\vx_{T-1}$.

Note that since each $\fieldn_j$ is independent and $\fieldn_j \sim \cN(0, \sigma^2)$ $\implies \vfieldn \sim \cN(\vzero, \sigma^2\vI)$\\
$\implies \vV\vfieldn \sim \cN(\vzero, \sigma^2\vV^\s{T}\vV)$. 

Therefore, using Lemma~\ref{lemma:hanson-wright}, we have with probability $\geq 1-\hp_f$,
\begin{align*}
    \|\vV\vfieldn\|^2 &\leq \s{Tr}(\sigma^2\vV^\s{T}\vV) + 2\|\sigma^2\vV^\s{T}\vV\|_\s{F}\sqrt{\log(1/\hp_f)} + 2\|\sigma^2\vV^\s{T}\vV\|\log(1/\hp_f) \\
    &\leq \sigma^2\s{Tr}(\vV^\s{T}\vV) + 2\sigma^2\|\vV^\s{T}\vV\|_\s{F}\sqrt{\log(1/\hp_f)} + 2\sigma^2\s{Tr}(\vV^\s{T}\vV)\log(1/\hp_f) \numberthis\label{eq:var-2-spec-trace-from}\\
    &= \sigma^2\s{Tr}(\vV^\s{T}\vV) + 2\sigma^2\sqrt{\s{Tr}(\vV^\s{T}\vV\vV^\s{T}\vV)\log(1/\hp_f)} + 2\sigma^2\s{Tr}(\vV^\s{T}\vV)\log(1/\hp_f)\\
    &\leq \sigma^2\s{Tr}(\vV^\s{T}\vV) + 2\sigma^2\sqrt{\|\vV^\s{T}\vV\|\s{Tr}(\vV^\s{T}\vV)\log(1/\hp_f)} + 2\sigma^2\s{Tr}(\vV^\s{T}\vV)\log(1/\hp_f)\\
    &= \sigma^2\s{Tr}(\vV^\s{T}\vV) + 2\sigma^2\s{Tr}(\vV^\s{T}\vV)\sqrt{\log(1/\hp_f)} + 2\sigma^2\s{Tr}(\vV^\s{T}\vV)\log(1/\hp_f)\\
    &= \sigma^2\s{Tr}(\vV^\s{T}\vV)(1+2\sqrt{\log(1/\hp_f)}+2\log(1/\hp_f))\\
    &\leq 5\sigma^2\s{Tr}(\vV^\s{T}\vV)\log(1/\hp_f),\numberthis\label{eq:var-2-trace}
\end{align*}
where in \eqref{eq:var-2-spec-trace-from}, we have used the fact that $\|\vM\| \leq \s{Tr}(\vM)$ and $\|\vM\|_\s{F} = \sqrt{\s{Tr}(\vM^\s{T}\vM)}$ for a symmetric matrix $\vM$.

Our aim is now to compute the value:
\begin{align*}
    \s{Tr}(\vV^\s{T}\vV) &= \|\vx_{T-1}\|^2 + \sum_{j=1}^{T-1} \|\vB_{T-1}\vB_{T-2}\dots \vB_{j}\vx_{j-1}\|^2\\
    &= \|\vx_{T-1}\|^2 + \sum_{\tau=1}^{T-1} \|(\prod_{j=\tau}^{T-1} \vB_{j})\vx_{j-1}\|^2\\
    &= \|\vx_{T-1}\|^2 + \sum_{\tau=1}^{\widetilde{\tau}} \|(\prod_{j=\tau}^{T-1} \vB_{j})\vx_{j-1}\|^2 + \sum_{\tau=\widetilde{\tau}+1}^{T-1} \|(\prod_{j=\tau}^{T-1} \vB_{j})\vx_{j-1}\|^2.\numberthis\label{eq:var-2-vTv}
\end{align*}
Using Lemma~\ref{lemma:ssgd-markov},  we have w.p. $\geq 1-T \hp_{var2, max}$,
\begin{align*}
    \sum_{\tau=1}^{\widetilde{\tau}} \|(\prod_{j=\tau}^{T-1} \vB_{j})\vx_{j-1}\|^2 &\leq  \sum_{\tau=1}^{\widetilde{\tau}} \frac{R_x^2}{\hp_{var2, \min}}e^{-\eta \mu \cdot (T-\tau)}, \numberthis \label{eq:var-2-normlem1}    
\end{align*}
where $\hp_{var2, \min}$ and $\hp_{var2, \max}$ are the minimum and maximum $\hp$'s respectively across all the Lemma~\ref{lemma:ssgd-markov} invocations.

Also, for any $\tau$, we have w.p. $\geq 1 - (T-\tau)\hp_{\vB}$,
\begin{align*}
    \|(\prod_{j=\tau}^{T-1} \vB_{j})\vx_{j-1}\|^2 &\leq \|\prod_{j=\tau}^{T-1} \vB_{j}\|^2\|\vx_{j-1}\|^2 \leq \|\vx_{j-1}\|^2 \leq R_x^2\log^{2a}(1/\hp_{\vB}), \numberthis\label{eq:var-2-normlem2}
\end{align*}
where we again use the fact that since $\eta \leq \frac{1}{R_x^2\log^{2a}(1/\hp_{\vB})}$ w.p. $\geq 1 - \hp_{\vB}$,
\begin{align*}
    \vB_j = \vI - \eta \vx_j\vx_j^\s{T} \preceq \vI \implies \|\vB_j\| \leq \|\vI\| = 1.    
\end{align*}
Using \eqref{eq:var-2-normlem1} and \eqref{eq:var-2-normlem2} in \eqref{eq:var-2-vTv}, we have w.p. at least $1-T\hp_{var2, \max} - (T-\widetilde{\tau})\hp_{\vx, \max} - T\hp_{\vB} \geq 1-T(\hp_{var2, \max} + \hp_{\vx, \max} + \hp_{\vB})$,
\begin{align*}
    \s{Tr}(\vV^\s{T}\vV) &= \|\vx_{T-1}\|^2 + \sum_{\tau=1}^{\widetilde{\tau}} \|(\prod_{j=\tau}^{T-1} \vB_{j})\vx_{j-1}\|^2 + \sum_{\tau=\widetilde{\tau}+1}^{T-1} \|(\prod_{j=\tau}^{T-1} \vB_{j})\vx_{j-1}\|^2\\
    &\leq R_x^2\log^{2a}(1/\hp_{\vx, \min})+ \sum_{\tau=1}^{\widetilde{\tau}} \frac{R_x^2}{\hp_{var2, \min}}e^{-\eta \mu \cdot (T-\tau)} + \sum_{\tau=\widetilde{\tau}+1}^{T-1} R_x^2\log^{2a}(1/\hp_{\vx, \min})\\
    &\leq R_x^2\log^{2a}(1/\hp_{\vx, \min}) \\
    &\qquad + \frac{R_x^2}{\hp_{var2, \min}} e^{-\eta \mu(T-\widetilde{\tau})} \frac{1}{1-e^{-\eta\mu}} + R_x^2\log^{2a}(1/\hp_{\vx, \min})(T-\widetilde{\tau}-1)\\
    &\leq R_x^2\log^{2a}(1/\hp_{\vx, \min})(T-\widetilde{\tau}) + \frac{3R_x^2}{\hp_{var2, \min}} e^{-\eta \mu(T-\widetilde{\tau})}. 
\end{align*}
Again, selecting $T-\widetilde{\tau}= \frac{101\log(1/\hp_{var2, \min})}{\eta \mu}$, we have w.p. $\geq 1-T(\hp_{var2, \max} + \hp_{\vx, \max}  + \hp_{\vB})$,
\begin{equation}
    \s{Tr}(\vV^\s{T}\vV)\leq R_x^2\cdot \Big(\log^{2a}(1/\hp_{\vx, \min})\frac{101}{\eta \mu}\log (1/\hp_{var2, \min}) + 3\hp_{var2, \min}^{100}\Big).
\end{equation}
Using this in \eqref{eq:ssgd-conc-var1-bound}, we have w.p. $\geq 1-T(\hp_{var2, \max} + \hp_{\vx, \max} + \hp_{\vB}) - \hp_f$,
\begin{align*}
    \|\vV\vfieldn\|^2  &\leq 5\sigma^2R_x^2\cdot \Big(\log^{2a}(1/\hp_{\vx, \min})\frac{101}{\eta \mu}\log (1/\hp_{var2, \min}) + 3\hp_{var2, \min}^{100}\Big)\log(1/\hp_f). \numberthis\label{eq:ssgd-conc-var2-bound}
\end{align*}

Combining \eqref{eq:ssgd-conc-var1-bound} and \eqref{eq:ssgd-conc-var2-bound} in \eqref{var-conc-split}, we obtain the overall variance bound w.p. at least $1-T(2\hp_{var, \max} + \hp_{\vx, \max} + \hp_{\vB}) - \hp_f - \hp_{tr}$, \\
$\|\eta(\vnu_{T-1} + \vB_{T-1}\vnu_{T-2}  + \dots + \vB_{T-1}\dots \vB_1\vnu_0)\|^2$
\begin{align*}
    &\leq 2\eta^2\Big(\|\vxi_{T-1} +  \vB_{T-1}\vxi_{T-2}  + \dots + \vB_{T-1}\dots \vB_1\vxi_{0}\|^2\\
    &\quad + 4\zeta^2\alpha^2\|\vg_{T-1} +  \vB_{T-1}\vg_{T-2}  + \dots + \vB_{T-1}\dots \vB_1\vg_{0}\|^2\Big) \\
    &\leq 2\eta^2\Big(\Big(5\sigma^2R_x^2\cdot \Big(\log^{2a}(1/\hp_{\vx, \min})\frac{101}{\eta \mu}\log (1/\hp_{var, \min}) + 3\hp_{var, \min}^{100}\Big)\log(1/\hp_f)\Big)\\
    &\quad + 4\zeta^2\alpha^2d\log  (1/\hp_{tr}) \Big(\frac{101}{\eta \mu}\log (1/\hp_{var, \min}) + 3\hp_{var, \max}^{100}\Big)\Big) \\
    &= 2\eta^2\Big(\Big(5\sigma^2R_x^2 \log (1/\hp_{var, \min}) \log^{2a}(1/\hp_{\vx, \min})\log(1/\hp_f) \\
    &\quad + 4\zeta^2\alpha^2d\log  (1/\hp_{tr})\log (1/\hp_{var, \min})\Big)\frac{101}{\eta \mu}\\
    &\quad + \Big(5\sigma^2R_x^2\log(1/\hp_f) + 4\zeta^2\alpha^2d\log  (1/\hp_{tr})\Big)3\hp_{var, \max}^{100}\Big),\numberthis\label{eq:ssgd-conc-var-bound}
\end{align*}
where $\hp_{var, \max} = \max(\hp_{var1, \max}, \hp_{var2, \max})$ and $\hp_{var, \min} = \max(\hp_{var1, \max}, \hp_{var2, \max})$.

Setting all the above $\hp$'s as $\widetilde{\hp}$,  and for $\widetilde{\hp}\leq 1/{\eta \mu}$, we get w.p. at least $1-(4T+2)\widetilde{\hp}$,
\begin{align}
\resizebox{\textwidth}{!}{
    $\|\eta(\vnu_{T-1} + \vB_{T-1}\vnu_{T-2}  + \dots + \vB_{T-1}\dots \vB_1\vnu_0)\|^2\leq \frac{210\eta}{\mu}\Big(5\sigma^2R_x^2\log^{2a+2} (1/\widetilde{\hp}) + 4\zeta^2\alpha^2d\log^2(1/\widetilde{\hp})\Big).$}\label{eq:ssgd-var-conc}
\end{align}
Using the Bias bound \eqref{eq:ssgd-bias-conc} and Variance bound \eqref{eq:ssgd-var-conc} in \eqref{eq:ssgd-bias-var-conc-split} and setting the additional $\hp_{\vB}$ arising from the Bias bound as well to $\widetilde{\hp}$, we have w.p. $\geq 1-(5T+2)\widetilde{\hp}$,
\begin{align*}
    \|\vw_{T} - \vw^*\|^2 &\leq 2\Big(\|\vw^*\|^2+ \frac{210\eta}{\mu}\Big(5\sigma^2 R_x^2\log^{2a+2} (1/\widetilde{\hp}) + 4\zeta^2\alpha^2d\log^2(1/\widetilde{\hp})\Big)\Big),
\end{align*}
Setting $\widetilde{\hp}= \frac{1}{\sf Poly(N)}$, we get w.p. $\geq 1-\frac{1}{\sf Poly(N)}$, 
\begin{align}
    \|\vw_{T} - \vw^*\| &\leq 2\|\wo\|+2\sqrt{\frac{1050 \eta}{\mu}} R_x \sigma \log^{a+1}N + 2 \sqrt{\frac{210 \eta \cdot d}{\mu}} \cdot 2\zeta\alpha \log N. \label{eq:sgd-norm-conc-bound}
\end{align}
Using \eqref{eq:sgd-norm-conc-bound} in \eqref{eq:ssgd-grad-bound-unplugged}, we get w.p. $\geq 1-\frac{1}{\sf Poly(N)}$,

\begin{align*}
    &\|\vx_T(\langle \vx_T, \vw_T\rangle - y_T)\| \leq  K_2R_x\log^{2a}N\Big(\sqrt{\|\vH\|_2}\|\vw_T - \vw^*\| + \sigma\Big)\\
    &\leq  K_2R_x\log^{2a}N\Big(2\sqrt{\|\vH\|_2}\Big(\|\wo\|+\sqrt{\frac{1050 \eta}{\mu}} R_x \sigma \log^{a+1}N + 2\sqrt{\frac{210 \eta \cdot d}{\mu}} \cdot \zeta\alpha \log N\Big) + \sigma\Big).
\end{align*}
Now, $\alpha = \alpha_{\epsilon, \delta, N}$. Hence, setting $\eta$, s.t., 
\begin{equation}
    1- 4 \sqrt{\frac{210 \eta \cdot d}{\mu}} \alpha_{\epsilon, \delta, N}  \cdot K_2 \sqrt{\|\vH\|_2}\cdot R_x\log^{2a+1} N \geq 1/2,
\end{equation}
i.e., if 
\begin{equation}
    \eta \leq \min\Big\{\frac{1}{64\cdot 210\cdot \log^{4a+2} N}\cdot  \frac{\mu}{K_2^2R_x^2 \|\vH\|}\cdot \frac{1}{d\alpha_{\epsilon, \delta, N}^2}, \frac{1}{4\cdot 1050\cdot \log^{2a+2} N}\cdot \frac{1}{R_x^2}\Big\},
\end{equation}
then w.p. $\geq 1-\frac{1}{\sf Poly(N)}$, 
$$\|\vx_T(\langle \vx_T, \vw_T\rangle - y_T)\| \leq \zeta,$$
and
\begin{equation}
\zeta = 4 K_2R_x\cdot \log^{2a}N\cdot \Big(\sqrt{\|\vH\|_2}\|\wo\| + \sqrt{\kappa}\sigma \Big).
\end{equation}
\end{proof}

\subsubsection{Proof of Lemma~\ref{lemma:sgd-zeta-infinity-utility}}
\begin{proof}
As in the proof of Lemma~\ref{lemma:ssgd-gradient-norm-bound}, let us define:
\begin{align*}
    \vB_t := \vI - \eta\vx_t\vx_t^\s{T},\ 
    \vxi_t := (y_t - \langle \vx_t, \vw^*\rangle)\vx_t = \fieldn_t\vx_t,\ 
    \vnu_t := \vxi_t - 2\zeta\alpha \vg_t.
\end{align*}
Since there is no clipping, the update rule can be written as:
\begin{align*}
    \vw_{t+1} - \vw^* &= \vw_{t} - \vw^* - \eta\text{ clip}_\zeta(\vx_t(\langle \vx_t, \vw_t\rangle - y_t)) - 2\eta\zeta\alpha \vg_t\\
    &= \vw_{t} - \vw^* - \eta\Big(\vx_t(\langle \vx_t, \vw_t\rangle - \langle \vx_t, \vw^* \rangle - \fieldn_t)) + 2\zeta\alpha \vg_t\Big)\\
    &= (\vI - \eta\vx_t\vx_t^\s{T})(\vw_t - \vw^*) + \eta_t\vxi_t - 2\eta\zeta\alpha \vg_t\\
    &= \vB_t(\vw_{t} - \vw^*) + \eta(\vxi_t - 2\zeta\alpha \vg_t)= \vB_t(\vw_{t} - \vw^*) + \eta\vnu_t. \numberthis\label{eq:ssgd-recursiveform}
\end{align*}
From recursion,
\begin{align*}
    \vw_t - \vw^* &= \vB_{t-1}\dots \vB_0(\vw_0 - \vw^*) + \eta(\vnu_{t-1} + \vB_{t-1}\vnu_{t-2} + \dots + \vB_{t-1}\dots \vB_1\vnu_0).\numberthis \label{eq:ssgd-closedform}
\end{align*}
Note that both the recursive form \eqref{eq:ssgd-recursiveform} and the expanded form \eqref{eq:ssgd-closedform} are similar to that of equation (1) in \cite{sgd} which analyzed a non-private version of SGD, except that we now have the term $\vnu_t = \vxi_t - 2\zeta\alpha\vg_t$ in place of $\vxi_t$. 

The above equation implies:
\begin{align*}
    \overline{\vw} - \vw^* &= \Big(\frac{2}{N}\sum_{\tau=N/2+1}^{N}\vw_{\tau}\Big) - \vw^*= \frac{2}{N}\sum_{\tau=N/2+1}^{N}(\vw_{\tau} - \vw^*)\\
    &= \frac{2}{N}\sum_{\tau=N/2+1}^{N}\Big(\vB_{\tau-1}\dots \vB_0(\vw_0 - \vw^*) + \eta(\vnu_{\tau-1} + \dots + \vB_{\tau-1}\dots \vB_1\vnu_0)\Big)\\
    &= \frac{2}{N}\sum_{\tau=N/2+1}^{N}\vB_{\tau-1}\dots \vB_0(\vw_0 - \vw^*) \\
    &\qquad + \frac{2\eta}{N}\sum_{\tau=N/2+1}^{N}\vnu_{\tau-1} + \vB_{\tau-1}\vnu_{\tau-2} + \dots + \vB_{\tau-1}\dots \vB_1\vnu_0.
\end{align*}
Therefore,
\begin{align*}
    &\Ee{(\vx, y) \sim \cD}{\|\overline{\vw} - \vw^*\|_\vH^2} \leq \Big(\sqrt{\Ee{(\vx, y) \sim \cD}{\Big\|\frac{2}{N}\sum_{\tau=N/2+1}^{N}\vB_{\tau-1}\dots \vB_0(\vw_0 - \vw^*)\Big\|_\vH^2}} \\
    &\quad + \sqrt{\Ee{(\vx, y) \sim \cD}{\Big\|\frac{2\eta}{N}\sum_{\tau=N/2+1}^{N}\vnu_{\tau-1} + \vB_{\tau-1}\vnu_{\tau-2} + \dots + \vB_{\tau-1}\dots \vB_1\vnu_0\Big\|_\vH^2}}\Big)^2,
\end{align*}
\begin{align*}
    \implies \cL(\overline{\vw})- \cL(\vw^*) &\leq \frac{1}{2}\Big(\underbrace{\sqrt{\Ee{(\vx, y) \sim \cD}{\|\overline{\vw} - \vw^*\|^2_{\vH}|\vnu_0 = \dots = \vnu_{N-1} = \vzero}}}_{\text{Bias Term}} \\
    &\quad + \underbrace{\sqrt{\Ee{(\vx, y) \sim \cD}{\|\overline{\vw} - \vw^*\|^2_{\vH}|\vw_0 = \vw^*}}}_{\text{Variance Term}}\Big)^2 \numberthis \label{eq:ssgd-loss-unplugged}
\end{align*}
where, 
\begin{align*}
    &\underbrace{\Ee{(\vx, y) \sim \cD}{\|\overline{\vw} - \vw^*\|^2_{\vH}|\vnu_0 = \dots = \vnu_{N-1} = \vzero}}_{\text{Bias Term}} := \Ee{(\vx, y) \sim \cD}{\Big\|\frac{2}{N}\sum_{\tau=N/2+1}^{N}\vB_{\tau-1}\dots \vB_0 \vw^*\Big\|_\vH^2}, \\
    &\underbrace{\Ee{(\vx, y) \sim \cD}{\|\overline{\vw} - \vw^*\|^2_{\vH}|\vw_0 = \vw^*}}_{\text{Variance Term}} := \Ee{(\vx, y) \sim \cD}{\Big\|\frac{2\eta}{N}\sum_{\tau=N/2+1}^{N}\vnu_{\tau-1} + \dots + \vB_{\tau-1}\dots \vB_1\vnu_0\Big\|_\vH^2}.
\end{align*}
This is because,
\begin{align*}
    \cL(\overline{\vw}) - \cL(\vw^*) = \frac{1}{2} \|\overline{\vw} - \vw^*\|^2_{\vH}.
\end{align*}
Following along the lines of \cite{sgd}, the Bias-Variance Analysis is as follows: \\\\
\textbf{Bias Term Analysis}
\begin{align*}
    &\resizebox{\textwidth}{!}{$\Ee{(\vx, y) \sim \cD}{\|\vw_t - \vw^*\|^2 | \vnu_0 = \dots = \vnu_{N-1} = \vzero}= \Ee{(\vx, y) \sim \cD}{\|\vw_t - \vw^*\|^2 | \vg_0\dots \vg_{N-1}=0, \vxi_0 \dots \vxi_{N-1} = \vzero} $}\\
    &= \Ee{(\vx, y) \sim \cD}{\|\vB_{t-1}(\vw_{t-1} - \vw^*)\|^2}= \Ee{(\vx, y) \sim \cD}{\|(\vI - \eta\vx_{t-1}\vx_{t-1}^\s{T})(\vw_{t-1} - \vw^*)\|^2}\\
    &= \Ee{(\vx, y) \sim \cD}{\|\vw_{t-1} - \vw^*\|^2} - \Ee{(\vx, y) \sim \cD}{2\eta\langle (\vw_{t-1} - \vw^*), \vx_{t-1}\vx_{t-1}^\s{T}(\vw_{t-1}, \vw^*) \rangle} \\
    &\qquad\qquad + \Ee{(\vx, y) \sim \cD}{\eta^2\langle (\vw_{t-1} - \vw^*), \|\vx_{t-1}\|^2\vx_{t-1}\vx_{t-1}^\s{T}(\vw_{t-1} - \vw^*) \rangle}\\
    &\leq \Ee{(\vx, y) \sim \cD}{\|\vw_{t-1} - \vw^*\|^2} - 2\eta\Ee{(\vx, y) \sim \cD}{\langle (\vw_{t-1} - \vw^*), \vH(\vw_{t-1} - \vw^*) \rangle} \\
    &\qquad\qquad + \eta^2 R_x^2\Ee{(\vx, y) \sim \cD}{\langle (\vw_{t-1} - \vw^*), \vH(\vw_{t-1} - \vw^*) \rangle}\\
    &= \Ee{(\vx, y) \sim \cD}{\|\vw_{t-1} - \vw^*\|^2} - \eta (2 - \eta R_x^2)  \Ee{(\vx, y) \sim \cD}{\langle (\vw_{t-1} - \vw^*), \vH(\vw_{t-1}, \vw^*) \rangle}\\
    &\leq \Ee{(\vx, y) \sim \cD}{\|\vw_{t-1} - \vw^*\|^2} - \eta \Ee{(\vx, y) \sim \cD}{\langle (\vw_{t-1} - \vw^*), \vH(\vw_{t-1}, \vw^*) \rangle} \numberthis \label{eq:ssgd-bias-approxplug}\\
    &\leq (1- \eta\mu)\Ee{(\vx, y) \sim \cD}{\|\vw_{t-1} - \vw^*\|^2}\leq (1- \eta\mu)^t\Ee{(\vx, y) \sim \cD}{\|\vw_0 - \vw^*\|^2} = e^{-\eta\mu t}\|\vw^*\|^2,
\end{align*}
where in \eqref{eq:ssgd-bias-approxplug}, we have used the fact that $\eta < \frac{1}{R_x^2}$.

The Bias Term's contribution to the overall risk is therefore:
\begin{align*}
    &\frac{1}{2}\Ee{(\vx, y) \sim \cD}{\|\overline{\vw} - \vw^*\|^2_\vH | \vnu_0 \dots \vnu_{N-1} = \vzero}\leq \frac{R_x^2}{2}\Ee{(\vx, y) \sim \cD}{\|\overline{\vw} - \vw^*\|^2 | \vnu_0  \dots  \vnu_{N-1} = \vzero}\\&= \frac{R_x^2}{2}\Ee{(\vx, y) \sim \cD}{\Big\|\frac{2}{N}\sum_{\tau=N/2+1}^{N}(\vw_{\tau} - \vw^*)\Big\|^2 | \vnu_0 \dots \vnu_{N-1} = \vzero}\\
    &\leq \frac{R_x^2}{N}\sum_{\tau=N/2+1}^{N}\Ee{(\vx, y) \sim \cD}{\|\vw_{\tau} - \vw^*\|^2 | \vg_0 \dots \vg_{N-1} = 0, \vxi_0 \dots \vxi_{N-1} = \vzero}\\&\leq \frac{R_x^2}{N}\sum_{\tau=N/2+1}^{N}e^{-\eta\mu \tau}\|\vw^*\|^2\leq \frac{R_x^2}{N}\cdot\frac{N}{2}e^{-\eta\mu (N/2+1)}\|\vw^*\|^2= \frac{R_x^2}{2}e^{-\eta\mu (N/2+1)}\|\vw^*\|^2. \numberthis\label{eq:ssgd-bias-risk}
\end{align*}
\textbf{Variance Term Analysis}\\
Now suppose $\vw_0 = \vw^*$. Define the covariance matrix:
\begin{align*}
    \vC_{t} := \Ee{(\vx, y) \sim \cD}{(\vw_t - \vw^*)(\vw_t - \vw^*)^\s{T} | \vw_0 = \vw^*}.
\end{align*}
Using the recursion $\vw_{t+1}- \vw^* = \vB_t(\vw_{t} - \vw^*) + \eta\vnu_t$ from \eqref{eq:ssgd-recursiveform}, we have $\vC_{t+1}$: 
\begin{align*}
    &\resizebox{\textwidth}{!}{$:= \Ee{(\vx, y) \sim \cD}{(\vw_{t+1} - \vw^*)(\vw_{t+1} - \vw^*)^\s{T}}= \Ee{(\vx, y) \sim \cD}{\Big(\vB_t(\vw_{t} - \vw^*) + \eta\vnu_t\Big)\Big(\vB_t(\vw_{t} - \vw^*) + \eta \vnu_t\Big)^\s{T}}$}\\
    &= \Ee{(\vx, y) \sim \cD}{\Big(\vB_t(\vw_t - \vw^*)\Big)\Big(\vB_t(\vw_t - \vw^*)\Big)^\s{T}} + \eta\Ee{(\vx, y) \sim \cD}{\Big(\vB_t(\vw_t - \vw^*)\Big)\vnu_t^\s{T}} \\
    &\quad + \eta\Ee{(\vx, y) \sim \cD}{\vnu_t\Big(\vB_t(\vw_t - \vw^*)\Big)^\s{T}} + \eta^2\Ee{(\vx, y) \sim \cD}{\vnu_t\vnu_t^\s{T}}\\
    &= \Ee{(\vx, y) \sim \cD}{\Big(\vB_t(\vw_t - \vw^*)\Big)\Big(\vB_t(\vw_t - \vw^*)\Big)^\s{T}} + \eta\Ee{(\vx, y) \sim \cD}{\vB_t(\vw_t - \vw^*)(\vxi_t^\s{T} - 2\zeta\alpha \vg_t^\s{T})} \\
    &\quad + \eta\Ee{(\vx, y) \sim \cD}{(\vxi_t - 2\zeta\alpha \vg_t)\Big(\vB_t(\vw_t - \vw^*)\Big)^\s{T}} + \eta^2\Ee{(\vx, y) \sim \cD}{(\vxi_t - 2\zeta\alpha \vg_t)(\vxi_t - 2\zeta\alpha \vg_t)^\s{T}}\\
    &= \Ee{(\vx, y) \sim \cD}{\Big(\vB_t(\vw_t - \vw^*)\Big)\Big(\vB_t(\vw_t - \vw^*)\Big)^\s{T}} + \eta\Ee{(\vx, y) \sim \cD}{\Big(\vB_t(\vw_t - \vw^*)\Big)\vxi_t^\s{T}} \\
    &\quad - 2\eta\zeta\alpha \Ee{(\vx, y) \sim \cD}{\Big(\vB_t(\vw_t - \vw^*)\Big)\vg_t^\s{T}} + \eta\Ee{(\vx, y) \sim \cD}{\vxi_t\Big(\vB_t(\vw_t - \vw^*)\Big)^\s{T}} \\
    &\quad - 2\eta\zeta\alpha\Ee{(\vx, y) \sim \cD}{ \vg_t\Big(\vB_t(\vw_t - \vw^*)\Big)^\s{T}}  + \eta^2\Ee{(\vx, y) \sim \cD}{(\vxi_t - 2\zeta\alpha \vg_t)(\vxi_t - 2\zeta\alpha \vg_t)^\s{T}}\\
    &= \Ee{(\vx, y) \sim \cD}{\Big(\vB_t(\vw_t - \vw^*)\Big)\Big(\vB_t(\vw_t - \vw^*)\Big)^\s{T}}+ \eta^2\Ee{(\vx, y) \sim \cD}{(\vxi_t - 2\zeta\alpha \vg_t)(\vxi_t - 2\zeta\alpha \vg_t)^\s{T}}\numberthis\label{eq:ssgd-variance-term10}\\
    &=\Ee{(\vx, y) \sim \cD}{\Big(\vB_t(\vw_t - \vw^*)\Big)\Big(\vB_t(\vw_t - \vw^*)\Big)^\s{T}} + \eta^2\Ee{(\vx, y) \sim \cD}{(\vxi_t - 2\zeta\alpha \vg_t)(\vxi_t - 2\zeta\alpha \vg_t)^\s{T}},
\end{align*}
where in \eqref{eq:ssgd-variance-term10} we have used the fact that $\Ee{(\vx_t, y_t) \sim \cD}{\vxi_t} = \vzero$ and $\E{\vg_t} = \vzero$ has been sampled at each step independently of all other terms. Continuing to expand the above expression gives: 
\begin{align*}
    \vC_{t+1} &= \Ee{(\vx, y) \sim \cD}{\Big(\vB_t(\vw_t - \vw^*)\Big)\Big(\vB_t(\vw_t - \vw^*)\Big)^\s{T}} \\
    &\quad \resizebox{0.9\textwidth}{!}{$+ \eta^2(\Ee{(\vx, y) \sim \cD}{\vxi_t\vxi_t^\s{T}} - 2\zeta\alpha\Ee{(\vx, y) \sim \cD}{\vxi_t \vg_t^\s{T}} - 2\zeta\alpha\Ee{(\vx, y) \sim \cD}{\vg_t\vxi_t^\s{T}} + 4\zeta^2\alpha^2\Ee{(\vx, y) \sim \cD}{\vg_t\vg_t^\s{T}})$}\\
    &= \Ee{(\vx, y) \sim \cD}{\vB_t(\vw_t - \vw^*)(\vw_t - \vw^*)^\s{T}\vB_t^\s{T}} + \eta^2(\vSigma - \vzero - \vzero + 4\zeta^2\alpha^2\vI) \numberthis\label{eq:ssgd-variance-ind}\\
    &= \Ee{(\vx, y) \sim \cD}{(\vI - \eta\vx_t\vx_t^\s{T})(\vw_t - \vw^*)(\vw_t - \vw^*)^\s{T}(\vI - \eta\vx_t\vx_t^\s{T})} + \eta^2(\vSigma + 4\zeta^2\alpha^2\vI)\\
    &= \Ee{(\vx, y) \sim \cD}{(\vw_t - \vw^*)(\vw_t - \vw^*)^\s{T}} - \eta \Ee{(\vx, y) \sim \cD}{(\vx\vx^\s{T})(\vw_t - \vw^*)(\vw_t - \vw^*)^\s{T}} \\
    &\quad - \eta\Ee{(\vx, y) \sim \cD}{(\vw_t - \vw^*)(\vw_t - \vw^*)^\s{T}(\vx\vx^\s{T})} + \eta^2\Ee{(\vx, y) \sim \cD}{(\vx\vx^\s{T})(\vw_t - \vw^*)(\vx\vx^\s{T})} \\
    &\quad + \eta^2(\vSigma + 4\zeta^2\alpha^2\vI)\\
    &= \vC_t - \eta \vH\vC_t - \eta \vC_t\vH + \eta^2\Ee{(\vx, y) \sim \cD}{\vx^\s{T}\vC_t\vx\vx\vx^\s{T}} + \eta^2(\vSigma + 4\zeta^2\alpha^2\vI),
\end{align*}
where in \eqref{eq:ssgd-variance-ind} we have again used the fact that $\E{\vg_t} = \vzero$ has been sampled independently at each step. Comparing this with equation (2) of \cite{sgd}, we find that the term $\gamma^2\vSigma$ has been replaced with $\eta^2(\vSigma + 4\zeta^2\alpha^2\vI)$ where $\vSigma = \Ee{(\vx_t, y_t) \sim \cD}{\vxi_t\vxi_t^\s{T}}$. Therefore the bound for the Variance term follows almost similarly from the paper, except that $\vSigma$ gets replaced with $\vSigma + 4\zeta^2\alpha^2\vI$. 

Using Lemma \ref{lemma:ssgd-variance}, we have the Variance Term contribution as
\begin{align*}
    \frac{1}{2}&\Ee{(\vx, y) \sim \cD}{\|\overline{\vw} - \vw^*\|^2_\vH |  \vw_0 = \vw^*} \leq \frac{\s{Tr}(\vC_\infty)}{\eta(N-(N/2+1)+1)} \\
    &\qquad\qquad\qquad\leq \frac{1}{N}\frac{\eta R_x^2}{1-\eta R_x^2} d\|\vSigma + 4\zeta^2\alpha^2\vI\|_\vH + \frac{1}{N}\s{Tr}(\vH^{-1}(\vSigma+4\zeta^2\alpha^2\vI)).\numberthis\label{eq:ssgd-variance-risk}
\end{align*}
Using \eqref{eq:ssgd-bias-risk} and \eqref{eq:ssgd-variance-risk} in \eqref{eq:ssgd-loss-unplugged} in:
\begin{align*}
   \cL(\overline{\vw}) - \cL(\vw^*) &\leq \frac{1}{2}\Big(\sqrt{\Ee{(\vx, y) \sim \cD}{\|\overline{\vw} - \vw^*\|^2_{\vH}|\vnu_0 = \dots = \vnu_{N-1} = 0}} \\
   &\qquad \qquad + \sqrt{\Ee{(\vx, y) \sim \cD}{\|\overline{\vw} - \vw^*\|^2_{\vH}|\vw_0 = \vw^*}}\Big)^2
\end{align*}
gives us the risk bound as:
\begin{align*}
    &\leq \frac{1}{2}\Big(\sqrt{\frac{R_x^2}{2}e^{-\eta\mu (\frac{N}{2}+1)}\|\vw^*\|^2} + \sqrt{\frac{1}{N}\frac{\eta R_x^2}{1-\eta R_x^2} d\|\vSigma + 4\zeta^2\alpha^2\vI\|_\vH + \frac{1}{N}\s{Tr}(\vH^{-1}(\vSigma+4\zeta^2\alpha^2\vI))}\Big)^2\\
    &\leq \frac{R_x^2}{2}e^{-\eta\mu (\frac{N}{2}+1)}\|\vw^*\|^2 + \frac{1}{N}\frac{\eta R_x^2}{1-\eta R_x^2} d\|\vSigma + 4\zeta^2\alpha^2\vI\|_\vH + \frac{1}{N}\s{Tr}(\vH^{-1}(\vSigma+4\zeta^2\alpha^2\vI)).
\end{align*}
This gives us the required result.
\end{proof}

\section{Missing Proofs from Section 3.2}\label{app:ambsgd}
\subsection{Proof of Privacy Guarantees}

\subsubsection{Proof of Lemma~\ref{lemma:dp-stat}}
\label{app:ambsgd-privacy}

\begin{proof}
Since each of the computation in Step~\ref{line:pop} of \dpstat is of sensitivity one and the DP noise variance is $\lceil\log_2(B/\Delta)\rceil\alpha^2$, each step is $\Big(\rho_i = \frac{1}{2\lceil\log_2(B/\Delta)\rceil\alpha^2}\Big)$-zCDP. Since we perform at most $\lceil \log_2(B/\Delta)\rceil$ of such computations, by standard zCDP property of Gaussian mechanism, and its corresponding composition property~\citep{bun2016concentrated}, the overall $\rho = \sum_{i = 1}^{\lceil\log_2(B/\Delta)\rceil} \rho_i = \frac{1}{2\alpha^2}$ and thus the privacy guarantee follows immediately.

To prove the utility guarantee, first note that by standard concentration inequality for Gaussian distribution and by union bound, it follows that w.p. at least $1-\beta$, none of the noise added in $\lceil\log_2(B/\Delta)\rceil$ iterations of Step~\ref{line:pop} exceeds $\Gamma=\alpha\sqrt{2\log (B/\Delta)\log(\log(B/\Delta)/\beta)}$.  We will condition the rest of the proof on this event.

For any value of $\gamma_i$ in iteration $i$ of Step~\ref{line:pop} of \dpstat, $|c_{\sf priv}-c|\leq \Gamma \implies c - \Gamma \leq c_{\sf priv} \leq c + \Gamma$. Now note that, 1) if $c_{\sf priv} \geq s$ then $c + \Gamma \geq s$, and 2) if $c_{\sf priv} \leq s$ then $c - \Gamma \leq s \implies c \leq s$ since the true count can never exceed s. Combining the two cases gives $s-\Gamma\leq c\leq s$. 

Furthermore, in the doubling search of Algorithm \dpstat,  each choice for $\gamma_i$ is of the form $2^{i+1}\Delta$. Therefore, if the loop breaks out at iteration $i^*$, there exists a $\gamma\in\left[2^{i^*-1}\Delta,2^{i^*}\Delta\right]$ s.t. $ s-\Gamma\leq\left|\left\{|\ip{\vx_i}{\vw}-y_i|\leq \gamma:i\in\{0,\ldots,s\}\right\}\right|\leq s$. Hence, error in $\gamma_{\sf priv}$ estimation can only be off by a factor of two, and an additional discretization error of $\Delta$. This implies the utility guarantee.
\end{proof}

\subsubsection{Proof of Theorem~\ref{thm:ambsgd-privacy}}
\begin{proof}
Our analysis will broadly involve computing the Zero Mean Concentrated Differential Privacy (zCDP) parameters and then using them to compute the Approximate Differential Privacy parameters. Update Step 10 of Algorithm \dpambssgd without the additive Gaussian noise is:
\begin{align*}
    \vw_{t+1} \gets \vw_{t} - \frac{\eta}{b} \sum_{i=0}^{b-1}\text{ clip}_{\zeta_t}(\vx_{\tau(t)+i}(\langle \vx_{\tau(t)+i}, \vw_{t}\rangle - y_{\tau(t)+i})),
\end{align*}
where $\text{clip}_\zeta(\vnu) = \vnu\cdot\max\Big\{1, \frac{\zeta}{\|\vnu\|_2}\Big\}$. Therefore, the local $L_2$ sensitivity of the $\vw_{t+1}$ due to a sample difference in the $\tau^\s{th}$ batch is $\Delta_2 = \frac{2\eta\zeta_t}{b}$.

Since $\vg_{t} \sim \cN(\vzero, \vI_{d \times d})$, the above step is $\Big(\rho_{t, 1} = \frac{\Delta_2^2}{2\cdot\frac{4\eta^2\zeta_t^2\alpha^2}{b^2}} = \frac{1}{2\alpha^2}\Big)$-zCDP since DP noise standard deviation is $\eta\frac{2\zeta_t\alpha}{b}$. (Proposition 1.6 of \cite{bun2016concentrated}) and from Lemma~\ref{lemma:dp-stat}, Step 10 is $\Big(\rho_{t, 2} = \frac{1}{2\alpha^2}\Big)$-zCDP. Therefore by composition, each iteration step is $\Big(\rho_t = \rho_{t,1} + \rho_{t,2} = \frac{1}{\alpha^2}\Big)$-zCDP. Observe that $\rho_{t}$ is a constant given a fixed value of $\alpha$.

Since each data sample $(\vx_i, y_i)$ $\forall$ $i \in [N]$ appears in exactly one mini-batch and the algorithm \dpambssgd takes only a single pass over the entire data, by parallel composition of zCDPs, the overall $\rho$ for \dpambssgd is given by $\rho = \max_{t \in [T]} \rho_{t} = \frac{1}{\alpha^2}$.

Recall $\rho$-zCDP  for an algorithm is equivalent to obtaining a $(\mu,\mu\rho)$-Renyi differential privacy (RDP)~\citep{mironov2017renyi} guarantee. In the following, we will optimize for $\mu\in[1,\infty)$ and demonstrate that for the choice of the noise multiplier $\alpha$ mentioned in the theorem statement satisfies $(\epsilon,\delta)$-DP, which would conclude the proof. Our analysis is similar to that of Theorem 1 in \cite{chien2021private}.

Note that $(\mu,\mu\rho)$-(RDP) $\implies$ $(\epsilon, \delta)$ Approximate Privacy where $\epsilon = \mu\rho + \frac{\log(1/\delta)}{\mu - 1}$ $\forall \mu > 1$. Also note that $\epsilon_{\min} = \rho + 2\sqrt{\rho\log(1/\delta)}$ is attained at $\frac{\text{d}\epsilon}{\text{d}\mu} = 0 \implies \mu = 1 + \sqrt{\frac{\log(1/\delta)}{\rho}}$.

Consider a fixed $\epsilon$. Since we want to minimize $\alpha$ (which scales as $1/\sqrt{\rho}$), we need to compute the maximum permissible $\rho$ s.t. $\epsilon_{\min}(\rho) \leq \epsilon$. Since $\epsilon_{\min}(\rho)$ is an increasing function of $\rho$ (thus an increasing function of $\alpha$) and a second order polynomial in $\sqrt{\rho}$ with root at $\sqrt{\rho} = \sqrt{\log (1/\delta) + \epsilon_{\min}} - \sqrt{\log(1/\delta)}$, the maximum is achieved at $\epsilon_{\min}(\rho) = \epsilon$. Therefore,
\begin{align*}
    \frac{1}{\alpha^2} &= (\sqrt{\log (1/\delta) + \epsilon} - \sqrt{\log(1/\delta)})^2= \frac{\epsilon^2}{(\sqrt{\log (1/\delta) + \epsilon} + \sqrt{\log(1/\delta)})^2}.
\end{align*}
Since the above value of $\alpha$ satisfies $(\epsilon, \delta)$-DP and 
\begin{align*}
    \frac{\epsilon^2}{(\sqrt{\log (1/\delta) + \epsilon} + \sqrt{\log(1/\delta)})^2} \geq \frac{\epsilon^2}{4(\log(1/\delta) + \epsilon)},
\end{align*}
choosing $\alpha\geq \frac{2\sqrt{\log(1/\delta)+\epsilon}}{\epsilon}$ ensures $(\epsilon, \delta)$-DP.

\end{proof}

\subsection{Proof of Risk Bounds}
\subsubsection{Auxiliary Technical Lemmas}
Here we first present a few results which will be used in proving our \dpambssgd risk bounds.

\begin{lemma}[From \cite{jain2018parallelizing}]\label{lemma:parallel-bias-risk}
For any learning rate $\eta \leq \frac{b}{R_x^2 + (b-1)\|\vH\|}$ after $N$ iterations of non-private fixed mini-batch SGD Algorithm with batch size $b$, the bias error for the final iterate and tail-averaged iterate respectively are  given by:
\begin{align*}
    \cL(\vw_{N}^{bias}) - \cL(\vw^*) &\leq \frac{\kappa}{2}(1-\eta\mu)^N(\cL(\vw_{0}) - \cL(\vw^*)),\\ 
    \cL(\overline{\vw}_{t:N}^{bias}) - \cL(\vw^*) &\leq \frac{2}{\eta^2N^2\mu^2}(1-\eta\mu)^{t}(\cL(\vw_{0}) - \cL(\vw^*)),
\end{align*}
where $\vw_0$ is the initialization of $\vw$ and $
    \overline{\vw}_{t:N} := \frac{1}{N}\sum_{t'=t}^{t+N-1}\vw_{t'}$.
\end{lemma}

\begin{lemma}[From \cite{jain2018parallelizing}]\label{lemma:parallel-variance-risk}
For any learning rate $\eta \leq \frac{b}{R_x^2 + (b-1)\|\vH\|}$ after $N$ iterations of non-private fixed mini-batch SGD Algorithm with batch size $b$ the variance error for the final iterate and tail-averaged iterate respectively are  given by:
\begin{align*}
    \cL(\vw_{N}^{variance}) - \cL(\vw^*) \leq \frac{\eta}{2b}\s{Tr}(\vH\cT_{b}^{-1}\vSigma\vI)),\  \ \  L(\overline{\vw}_{t:N}^{variance}) - \cL(\vw^*) \leq \frac{2}{Nb}\s{Tr}(\cT_b^{-1}\vSigma),
\end{align*}
where $\overline{\vw}_{t:N} := \frac{1}{N}\sum_{t'=t}^{t+N-1}\vw_{t'}$. Further, if we define,
\begin{align*}
    \resizebox{\textwidth}{!}{$\vA := \Big(\cH_{\cL} + \cH_{\cR} - \frac{\eta}{b}\cdot(b-1)\cH_{\cL}\cH_{\cR}\Big)^{-1}\vSigma,\ \ 
    \overline{\vw}_{t:N} := \frac{1}{N}\sum_{t'=t}^{t+N-1}\vw_{t'},\ \ 
    \eta \leq \frac{2b}{R_x^2\cdot\frac{d\|(\cH_{\cL} + \cH_{\cR})^{-1}\vSigma\|_2}{\s{Tr}((\cH_{\cL} + \cH_{\cR})^{-1}\vSigma)}},$}
\end{align*}
then, 
\begin{align*}
    \s{Tr}(\cT_b^{-1}\vSigma) &\leq \s{Tr}(\vA) + \frac{\frac{\eta R_x^2}{2b}d\|(\cH_{\cL} + \cH_{\cR})^{-1}\vSigma\|_2}{\Big(1 - \frac{\eta}{2b}\cdot(R_x^2 + (b-1)\|\vH\|_2)\Big)\Big(1 - \eta\frac{b-1}{2b}\|\vH\|_2\Big)} \leq 2\s{Tr}(\vH^{-1}\vSigma).
\end{align*}
\end{lemma}

\begin{lemma}\label{lemma:ambsgd-btbt}
    If $\eta \leq \frac{b}{R_x^2 + (b-1)\|\vH\|}$, $\tau(t)= t\cdot  (b+s)$ and $\vB_t := \Big(\vI - \frac{\eta}{b}\sum_{i=0}^{b-1} \vx_{\tau(t)+s+i} \vx_{\tau(t)+s+i}^\s{T}\Big)$, then $\forall j$, 
        $\Ee{(\vx, y) \sim \cD}{\vB_j^\s{T}\vB_j} \preceq \vI - \eta\vH.$
\end{lemma}

\begin{proof}
\begin{align*}
    &\Ee{(\vx, y) \sim \cD}{\vB_j^\s{T}\vB_j}= \Ee{(\vx, y) \sim \cD}{\Big(\vI - \frac{\eta}{b}\sum_{i=0}^{b-1} \vx_{\tau(j)+s+i} \vx_{\tau(j)+s+i}^\s{T}\Big)\Big(\vI - \frac{\eta}{b}\sum_{i=0}^{b-1} \vx_{\tau(j)+s+i} \vx_{\tau(j)+s+i}^\s{T}\Big)}\\
    &= \vI - \frac{2\eta}{b} \Ee{(\vx, y) \sim \cD}{\sum_{i=0}^{b-1} \vx_{\tau(j)+s+i} \vx_{\tau(j)+s+i}^\s{T}}  \\
    &\qquad \qquad+ \frac{\eta^2}{b^2}\Ee{(\vx, y) \sim \cD}{\sum_{i=0}^{b-1}\sum_{m=0}^{b-1} \vx_{\tau(j)+s+i} \vx_{\tau(j)+s+i}^\s{T}\vx_{\tau(j)+s+m} \vx_{\tau(j)+s+m}^\s{T}}\\
    &= \vI - \frac{2\eta}{b} b\vH + \frac{\eta^2}{b^2}\Big(b\Ee{(\vx, y) \sim \cD}{\|\vx\|^2\vx\vx^\s{T}} + b(b-1)\Big(\Ee{(\vx, y) \sim \cD}{\vx\vx^\s{T}}\Big)^2\Big)\\
    &\preceq \vI - 2\eta\vH + \frac{\eta^2}{b}\Big(R_x^2\vH + (b-1)\|\vH\|\vH\Big)\preceq \vI - \eta\vH\Big(2 - \frac{\eta}{b}(R_x^2 + (b-1)\|\vH\|)\Big). \numberthis\label{eq:mbsgd-markov-2}
\end{align*}
Now, $\eta \leq \frac{b}{R_x^2 + (b-1)\|\vH\|} \implies \frac{\eta}{b}(R_x^2 + (b-1)\|\vH\|) \leq 1$. Using this in \eqref{eq:mbsgd-markov-2} gives 
\begin{align*}
    \Ee{(\vx, y) \sim \cD}{\vB_j^\s{T}\vB_j} &\preceq \vI - \eta\vH.
\end{align*}

\end{proof}

\begin{lemma}\label{lemma:ambsgd-nutnut^T}
    If $\vnu_j = \frac{1}{b}\sum_{i=0}^{b-1} \vxi_{\tau(j)+s+i} - \frac{2\zeta_j\alpha}{b}\vg_j$, $\vxi_{\tau(t)+s + i} :=  \fieldn_{\tau(t)+s+i}\vx_{\tau(t)+s+i}$ and $\tau(t)= t\cdot  (b+s)$, then $\forall$ $j$
    \begin{align*}
        \Ee{(\vx, y) \sim \cD}{\vnu_{ j} \vnu_{ j}^\s{T}} &= \frac{1}{b}\vSigma + \frac{4\alpha^2}{b^2} \Ee{(\vx, y) \sim \cD}{\zeta_j^2}\vI.
    \end{align*}
\end{lemma}

\begin{proof}
The covariance of $\vnu_{ j}$ is given by $\Ee{(\vx, y) \sim \cD}{\vnu_{ j} \vnu_{ j}^\s{T}}$
\begin{align*}
    &= \Ee{(\vx, y) \sim \cD}{\Big(\frac{1}{b}\sum_{i=0}^{b-1}\vxi_{\tau(j)+s+i} - \frac{2\zeta_j\alpha}{b}\vg_{ j}\Big) \Big(\frac{1}{b}\sum_{i=0}^{b-1}\vxi_{\tau(j)+s+i} - \frac{2\zeta_j\alpha}{b}\vg_{ j}\Big)^\s{T}}\\
    &= \Ee{(\vx, y) \sim \cD}{\Big(\frac{1}{b}\sum_{i=0}^{b-1}\vxi_{\tau(j)+s+i}\Big) \Big(\frac{1}{b}\sum_{i=0}^{b-1}\vxi_{\tau(j)+s+i}\Big)^\s{T}} + \Ee{(\vx, y) \sim \cD}{\frac{4\zeta_j^2\alpha^2}{b^2} \vg_{ j} \vg_{ j}^\s{T}}\\
    &\quad - \Ee{(\vx, y) \sim \cD}{\Big(\frac{2\zeta_j\alpha}{b^2} \sum_{i=0}^{b-1}\vxi_{\tau(j)+s+i}\Big)\vg_{ j}^\s{T}} - \Ee{(\vx, y) \sim \cD}{\vg_{ j}\Big(\frac{2\zeta_j\alpha}{b^2} \sum_{i=0}^{b-1}\vxi_{\tau(j)+s+i}\Big)^\s{T}} \\
    &= \Ee{(\vx, y) \sim \cD}{\Big(\frac{1}{b}\sum_{i=0}^{b-1}\vxi_{\tau(j)+s+i}\Big) \Big(\frac{1}{b}\sum_{i=0}^{b-1}\vxi_{\tau(j)+s+i}\Big)^\s{T}} - \vzero \\
    &\quad - \vzero + \Ee{(\vx, y) \sim \cD}{\frac{4\zeta_j^2\alpha^2}{b^2} \vg_{ j} \vg_{ j}^\s{T}} \numberthis\label{eq:ambsgd-cov-lemma-rec1}\\
    &= \Ee{(\vx, y) \sim \cD}{\Big(\frac{1}{b}\sum_{i=0}^{b-1}\vxi_{\tau(j)+s+i}\Big)\Big(\frac{1}{b}\sum_{i=0}^{b-1}\vxi_{\tau(j) + s+k}\Big)^\s{T}} + \Ee{(\vx, y) \sim \cD}{\frac{4\zeta_j^2\alpha^2}{b^2} \vg_{ j}\vg_{ j}^\s{T}}\\
    &= \Ee{(\vx, y) \sim \cD}{\frac{1}{b^2}\sum_{i=0}^{b-1}\sum_{k=0}^{b-1}\vxi_{\tau(j)+s+i} \vxi_{\tau(j) + s+k}^\s{T}} + \frac{4\alpha^2}{b^2} \Ee{(\vx, y) \sim \cD}{\zeta_j^2}\vI\\
    &= \frac{1}{b}\vSigma + \frac{4\alpha^2}{b^2} \Ee{(\vx, y) \sim \cD}{\zeta_j^2}\vI, \numberthis\label{eq:ambsgd-cov-lemma-rec2}\\
\end{align*}
where in \eqref{eq:ambsgd-cov-lemma-rec1}, we have used the fact that $\Ee{(\vx, y) \sim \cD}{\vxi_{\tau(j)+s + i}\vxi_{\tau(j)+s + k}^\s{T}} = \vzero $ if $k \neq i$ due to the independence of the samples in \eqref{eq:ambsgd-cov-lemma-rec2} and the fact that $\Ee{(\vx, y) \sim \cD}{\vg_{j}} =\vzero$ has been sampled independently at each step.

\end{proof}

\subsubsection{Proof of Lemma~\ref{lemma:ambssgd-zeta-infinity-utility}}\label{app:ambsgd-utility}
\begin{proof}
Since there is no clipping, the $t^\s{th}$ iteration update is given by:
\begin{align*}
    \vw_{t+1} &= \vw_t -\frac{\eta}{b} \sum_{i=0}^{b-1} \vx_{\tau(t)+s+i}(\langle \vx_{\tau(t)+s+i}, \vw_t\rangle - y_{\tau(t)+s+i}) - \frac{2\eta\zeta_t\alpha}{b}\vg_t,
\end{align*}
where $\tau(t)= t\cdot  (b+s)$. 

The update step can be written as: 
\begin{align*}
    \vw_{t+1} - \vw^* &=
    \Big(\vI -\frac{\eta}{b} \sum_{i=0}^{b-1} \vx_{\tau(t)+s+i} \vx_{\tau(t)+s+i}^\s{T}\Big)(\vw_t - \vw^*) + \frac{\eta}{b} \sum_{i=0}^{b-1} \fieldn_{\tau(t)+s + i}\vx_{\tau(t)+s+i} - \frac{2\eta\zeta_t\alpha}{b} \vg_t\\
    \implies \vomega_{t+1} &= \Big(\vI - \frac{\eta}{b}\sum_{i=0}^{b-1} \vx_{\tau(t)+s+i} \vx_{\tau(t)+s+i}^\s{T}\Big)\vomega_t + \frac{\eta}{b}\sum_{i=0}^{b-1}\fieldn_{\tau(t)+s + i}\vx_{\tau(t)+s + i} - \frac{2\eta\zeta_t\alpha}{b}\vg_t\\
    &= \Big(\vI - \frac{\eta}{b}\sum_{i=0}^{b-1} \vx_{\tau(t)+s+i} \vx_{\tau(t)+s + i}^\s{T}\Big)\vomega_{t} + \frac{\eta}{b}\sum_{i=0}^{b-1}\vxi_{\tau(t)+s + i} - \frac{2\eta\zeta_t\alpha}{b}\vg_t\\
    &= \vB_t\vomega_t + \eta\Big(\frac{1}{b}\sum_{i=0}^{b-1}\vxi_{\tau(t)+s + i} - \frac{2\zeta_t\alpha}{b}\vg_t\Big)\\
    &= \vB_t\vomega_t + \eta\vnu_t,
\end{align*}
where
\begin{align*}
    \vomega_t &:= \vw_t - \vw^*,\ \ 
    \vB_t := \Big(\vI - \frac{\eta}{b}\sum_{i=0}^{b-1} \vx_{\tau(t)+s+i} \vx_{\tau(t)+s+i}^\s{T}\Big),\\
    \vxi_{\tau(t)+s + i} &:=  \fieldn_{\tau(t)+s+i}\vx_{\tau(t)+s+i},\ \ 
    \vnu_t := \frac{1}{b}\sum_{i=0}^{b-1}\vxi_{\tau(t)+s + i} - \frac{2\zeta_t\alpha}{b}\vg_t.
\end{align*}
Therefore,  $\Ee{(\vx, y) \sim \cD}{\|\vw_{t} - \vw^*\|^2_{\vH}}$
\begin{align*}
    &= \Ee{(\vx, y) \sim \cD}{\Big\|\vB_{t-1}(\vw_{t-1} - \vw^*) + \eta\Big(\frac{1}{b}\sum_{i=0}^{b-1}\vxi_{\tau(t-1)+s + i} - \frac{2\zeta_{t-1}\alpha}{b}\vg_{t-1}\Big)\Big\|^2_{\vH}}\\
    &= \Ee{(\vx, y) \sim \cD}{\|\vB_{t-1}(\vw_{t-1} - \vw^*)\|^2_{\vH}} \\
    &\quad + 2\Ee{(\vx, y) \sim \cD}{\eta\Big(\frac{1}{b}\sum_{i=0}^{b-1}\vxi_{\tau(t-1)+s + i} - \frac{2\zeta_{t-1}\alpha}{b}\vg_{t-1}\Big)^\s{T}\vH(\vB_{t-1}(\vw_{t-1} - \vw^*))} \\
    &\quad + \Ee{(\vx, y) \sim \cD}{\Big\|\eta\Big(\frac{1}{b}\sum_{i=0}^{b-1}\vxi_{\tau(t-1)+s + i} - \frac{2\zeta_{t-1}\alpha}{b}\vg_{t-1}\Big)\Big\|^2_{\vH}}\\
    &= \Ee{(\vx, y) \sim \cD}{\|\vB_{t-1}(\vw_{t-1} - \vw^*)\|^2_{\vH}} + 0 + \Ee{(\vx, y) \sim \cD}{\Big\|\eta\Big(\frac{1}{b}\sum_{i=0}^{b-1}\vxi_{\tau(t-1)+s + i} - \frac{2\zeta_{t-1}\alpha}{b}\vg_{t-1}\Big)\Big\|^2_{\vH}}\numberthis\label{eq:ambsgd-new-onestep-put1}. 
\end{align*}
That is, 
\begin{align*}
    &\Ee{(\vx, y) \sim \cD}{\|\vw_{t} - \vw^*\|^2_{\vH}}=  \Ee{(\vx, y) \sim \cD}{(\vw_{t-1} - \vw^*)^\s{T}\vB_{t-1}^\s{T}\vH\vB_{t-1
    }(\vw_{t-1} - \vw^*)} \\
    &\quad + \Ee{(\vx, y) \sim \cD}{\eta^2\Big(\frac{1}{b}\sum_{i=0}^{b-1}\vxi_{\tau(t-1)+s + i}\Big)^\s{T}\vH\Big(\frac{1}{b}\sum_{i=0}^{b-1}\vxi_{\tau(t-1)+s + i}\Big)}  \\
    &\quad - 2\Ee{(\vx, y) \sim \cD}{\eta^2\frac{2\zeta_{t-1}\alpha}{b}\Big(\frac{1}{b}\sum_{i=0}^{b-1}\vxi_{\tau(t-1)+s + i}\Big)^\s{T}\vH\vg_{t-1}} + \Ee{(\vx, y) \sim \cD}{\eta^2\frac{4\zeta_{t-1}^2\alpha^2}{b^2}\vg_{t-1}^\s{T}\vH\vg_{t-1}} \\
    &\leq (1-\eta\mu) \Ee{(\vx, y) \sim \cD}{\|\vw_{t-1} - \vw^*\|^2_{\vH}} + \frac{\eta^2}{b}\s{Tr}(\vH\vSigma) -  0 + \Ee{(\vx, y) \sim \cD}{\eta^2\frac{4\zeta_{t-1}^2\alpha^2}{b^2}}\s{Tr}(\vH) \numberthis\label{eq:ambsgd-new-onestep-put2}\\
    &\leq (1-\eta\mu) \Ee{(\vx, y) \sim \cD}{\|\vw_{t-1} - \vw^*\|^2_{\vH}} + \frac{\eta^2}{b}\s{Tr}(\vH\vSigma) \\
    &\quad + \Ee{(\vx, y) \sim \cD}{\eta^2\frac{4\alpha^2}{b^2}\cdot 3K_2^2R_x^2\log^{4a} N\Big(\|\vH\|\|\vw_{t-1} - \vw^*\|^2 + \sigma^2 + \Delta^2\Big)}\s{Tr}(\vH) \numberthis\label{eq:ambsgd-new-onestep-put3}\\
    &= \Big(1-\Big(\eta\mu - 12\eta^2\frac{\alpha^2}{b^2} K_2^2R_x^2\kappa\log^{4a}N\s{Tr}(\vH)\Big)\Big) \Ee{(\vx, y) \sim \cD}{\|\vw_{t-1} - \vw^*\|^2_{\vH}} \\
    &\quad + \frac{\eta^2}{b}\s{Tr}(\vH\vSigma) + 12\eta^2\frac{\alpha^2}{b^2} K_2^2R_x^2\log^{4a}N(\sigma^2 + \Delta^2) \s{Tr}(\vH), \numberthis\label{eq:ambsgd-new-onestep-put4}
\end{align*}
where in \eqref{eq:ambsgd-new-onestep-put1} and \eqref{eq:ambsgd-new-onestep-put2} we have used the fact that $\Ee{(\vx, y) \sim \cD}{\vxi_{\tau(t-1)+s+i}} = \vzero$, $\E{\vg_{t-1}} = \vzero$ is sampled independently of all the other terms at each step and Lemmas ~\ref{lemma:ambsgd-btbt} and ~\ref{lemma:ambsgd-nutnut^T}. In \eqref{eq:ambsgd-new-onestep-put3} we have used the value:
\begin{align*}
    \zeta_{t} &\leq R_xK_2\log^{2a} N\Big(\sqrt{\|\vH\|}\|\vw_{t}-\wo\|+ \sigma + \Delta\Big),
\end{align*}
where $\Delta = \frac{\|\vw^*\|_{\vH} + \sigma}{N^{100}}$ as well as the fact $\vI\mu \preceq \vH \implies \|\vH\|\|\vw_{t-1} - \vw^*\|^2 \leq \kappa\|\vw_{t-1} - \vw^*\|^2_{\vH}$.

Since $T\cdot (b+s) = N$, therefore if
\begin{align*}
    \frac{\eta\mu}{2} &\geq 12\eta^2\frac{\alpha^2}{b^2} K_2^2R_x^2\kappa\log^{4a}N\s{Tr}(\vH)\\
    \implies \Big(\frac{N}{T}-s\Big)^2 &\geq \frac{24\eta \alpha^2 K_2^2R_x^2\kappa\log^{4a}N\s{Tr}(\vH)}{\mu},
\end{align*}
then \eqref{eq:ambsgd-new-onestep-put4} gives $\Ee{(\vx, y) \sim \cD}{\|\vw_t - \vw^*\|^2_{\vH}}$ 
\begin{align*}
    &\leq (1-\eta\mu/2)\Ee{(\vx, y) \sim \cD}{\|\vw_{t-1} - \vw^*\|^2_{\vH}} + \frac{\eta^2}{b}\s{Tr}(\vH\vSigma) + 12\eta^2\frac{\alpha^2}{b^2} K_2^2R_x^2\log^{4a}N(\sigma^2 + \Delta^2) \s{Tr}(\vH)\\
    &\leq (1-\eta\mu/2)^t\|\vw_0 - \vw^*\|^2_{\vH} + \frac{2}{\eta\mu}\Big(\frac{\eta^2}{b}\s{Tr}(\vH\vSigma) + 12\eta^2\frac{\alpha^2}{b^2} K_2^2R_x^2\log^{4a}N(\sigma^2 + \Delta^2) \s{Tr}(\vH)\Big)\\
    &\leq e^{-\eta\mu t/2}\|\vw^*\|^2_{\vH} + \frac{2\eta}{\mu b}\s{Tr}(\vH\vSigma) + \frac{24\eta\alpha^2}{\mu b^2} K_2^2R_x^2\log^{4a} N(\sigma^2 + \Delta^2) \s{Tr}(\vH).
\end{align*}
Note that the Bias Term decay rate is $\eta\mu/2$. For $t = T/2 = \frac{N}{2(b+s)}$ and $T = c_1\kappa\log N$, we have
\begin{align*}
    \Ee{(\vx, y) \sim \cD}{\|\vw_{T/2} - \vw^*\|^2_{\vH}} &\leq \frac{\|\vw^*\|^2_\vH}{N^{\frac{\eta\mu}{4}c_1\kappa}} + \frac{2\eta}{\mu b}\s{Tr}(\vH\vSigma) + \frac{24\eta\alpha^2}{\mu b^2} K_2^2R_x^2\log^{4a} N(\sigma^2 + \Delta^2) \s{Tr}(\vH). \numberthis\label{eq:ambsgd-new-wt/2-value}
\end{align*}
Since $\zeta_{t} \leq R_xK_2\log^{2a} N\Big(\sqrt{\|\vH\|}\|\vw_{t}-\wo\|+ \sigma + \Delta\Big)$, the bound on $\Ee{(\vx, y) \sim \cD}{\zeta_t^2}$
\begin{align*}
    \Ee{(\vx, y) \sim \cD}{\zeta_t^2} &\leq 3 K_2^2R_x^2\log^{4a} N\Big(\kappa\Ee{(\vx, y) \sim \cD}{\|\vw_t- \vw^*\|^2_{\vH}} + \sigma^2 + \Delta^2\Big)\\
    &\leq 3 K_2^2R_x^2\log^{4a} N\Big(\kappa\Big(e^{-\eta\mu/2t}\|\vw^*\|^2_{\vH} + \frac{2\eta}{\mu b}\s{Tr}(\vH\vSigma) \\
    &\quad + \frac{24\eta\alpha^2}{\mu b^2} K_2^2R_x^2\log^{4a} N(\sigma^2 + \Delta^2) \s{Tr}(\vH)\Big) + \sigma^2 + \Delta^2\Big) \numberthis\label{eq:ambsgd-new-nujnuj-inequality1}
\end{align*}
is decreasing with $t$ \Big(w.p. $\geq 1-\frac{1}{\sf Poly(N)}$\Big).

Furthermore, from Lemma~\ref{lemma:ambsgd-nutnut^T},
\begin{align*}
    \Ee{(\vx, y) \sim \cD}{\vnu_t\vnu_t^\s{T}} = \frac{1}{b}\vSigma + \frac{4\alpha^2}{b^2} \Ee{(\vx, y) \sim \cD}{\zeta_t^2}\vI \;\; \forall t\numberthis\label{eq:ambsgd-new-nujnuj-inequality2}
\end{align*}
Thus, if we define $\zeta$ s.t.
\begin{align*}
    \zeta^2 &= \max \{\text{Upper-Bound}(\Ee{(\vx, y) \sim \cD}{\zeta_{T/2}^2}), \dots, \text{Upper-Bound}(\Ee{(\vx, y) \sim \cD}{\zeta_{T}^2})\}\\
    &= \text{Upper-Bound}(\Ee{(\vx, y) \sim \cD}{\zeta_{T/2}^2})\\
    &= 3 K_2^2R_x^2\log^{4a} N\Big(\kappa\Ee{(\vx, y) \sim \cD}{\|\vw_{T/2} - \vw^*\|^2_{\vH}} + \sigma^2 + \Delta^2\Big), \numberthis\label{eq:ambsgd-new-zeta-def}
\end{align*}
where in the last step we have used \eqref{eq:ambsgd-new-nujnuj-inequality1}, then using \eqref{eq:ambsgd-new-nujnuj-inequality2}, 
\begin{align*}
    \Ee{(\vx, y) \sim \cD}{\vnu_t\vnu_t^\s{T}} &= \frac{1}{b}\vSigma + \frac{4\alpha^2}{b^2} \Ee{(\vx, y) \sim \cD}{\zeta_t^2}\vI \\
    &\preceq \frac{1}{b}\vSigma + \frac{4\alpha^2}{b^2} \zeta^2 \vI\;\; \forall t \in \{T/2, \dots, T\}.\numberthis\label{eq:ambsgd-new-nujnuj-zeta-inequality}
\end{align*}
This implies that we can perform the tail-averaged iterate analysis by restarting the algorithm with the initial value of $\vw_0 = \vw_{T/2}$ obtained in \eqref{eq:ambsgd-new-wt/2-value} and replacing $\zeta_{T/2}^2$ with $\zeta$ as defined in \eqref{eq:ambsgd-new-zeta-def} for the remaining $T/2$ iterations.

Thus, now consider re-running the algorithm for $T/2$ iterations with the initialization $\vw_0 = \vw_{T/2}$ and constant $\zeta_t = \zeta$ defined above. The $t^\s{th}$ iteration update is given by:
\begin{align*}
    \implies \vomega_{t+1} &= \vB_t\vomega_t + \eta\Big(\frac{1}{b}\sum_{i=0}^{b-1}\vxi_{\tau(t)+s + i} - \frac{2\zeta\alpha}{b}\vg_t\Big)= \vB_t\vomega_t + \eta\vnu_t, \numberthis\label{eq:ambsgd-new-bias-variance-recursive}
\end{align*}
where, $\vnu_t := \frac{1}{b}\sum_{i=0}^{b-1}\vxi_{\tau(t)+s + i} - \frac{2\zeta\alpha}{b}\vg_t.$

Expanding the recursion thus gives
\begin{align*}
    \vomega_{t+1} &= \vB_t\vomega_{t-1} + \eta\vnu_t= \vB_t\vB_{t-1}\dots\vB_0\vomega_0 + \eta(\vnu_t + \vB_t\vnu_{t-1} + \dots + \vB_t\dots\vB_1\vnu_0)\\
    &= \vB_t\vB_{t-1}\dots\vB_{0}\vomega_{0} + \eta( \vnu_t + \sum_{k=0}^{t-1} (\vB_t\dots\vB_{t-k}\vnu_{t-1-k}))\\
    &:= \vQ_{(0,t)}\vomega_{0} + \eta\sum_{j=0}^{t}\vQ_{(j+1, t)}\vnu_j \numberthis\label{eq:ambsgd-new-bias-variance-expanded}:= \vomega_{t+1}^{bias} + \vomega_{t+1}^{variance},
\end{align*}
where, $    \vQ_{(j,t)} := \vB_t\vB_{t-1}\dots\vB_j \text{  s.t.  } \vQ_{(j,t)} := \vI \text{  if  } j> t,\ $ $    \vomega_t^{bias} := \vQ_{(0,t-1)}\vomega_{0},\ \ 
    \vomega_t^{variance} := \eta\sum_{j=0}^{t}\vQ_{(j+1, t)}\vnu_{\ell, j}.$
    
Now, note that both the recursive form \eqref{eq:ambsgd-new-bias-variance-recursive} and the expanded form \eqref{eq:ambsgd-new-bias-variance-expanded} are the same as that obtained in equation (12) in \cite{jain2018parallelizing} which analyzed a non-private version of fixed mini-batch SGD Algorithm, except that we now have the term $\vnu_t = \frac{1}{b}\sum_{i=0}^{b-1}\vxi_{\tau(t)+s + i} - \frac{2\zeta\alpha}{b}\vg_t$ in place of $\vzeta_{j,b}$. 

Following along similar lines, we denote the tail averaged iterate as (Note the change in indexing since we are rerunning the algorithm with $\vw_0 = \vw_{T/2}$ for only $T/2$ iterations):
\begin{align*}
    \overline{\vw} &:= \frac{2}{T}\sum_{t'=1}^{T/2}\vw_{t'}.\\
\end{align*}
Then we have a similar tail averaged version of $\vomega$:
\begin{align*}
    \overline{\vomega} &= \frac{2}{T}\sum_{t'=1}^{T/2}\vomega_{t'}= \frac{2}{T}\sum_{t'=1}^{T/2}(\vomega_{t'}^{bias} + \vomega_{t'}^{variance})= \underbrace{\frac{2}{T}\sum_{t'=1}^{T/2}\vomega_{t'}^{bias}}_{\text{Tail Averaged Bias}} + \underbrace{\frac{2}{T}\sum_{t'=1}^{T/2}\vomega_{t'}^{variance}}_{\text{Tail Averaged Bias}}\\&:= \overline{\vomega}^{bias} + \overline{\vomega}^{variance}.
\end{align*}
The overall error can then be bounded as:
\begin{align*}
    \cL(\vw_t) - \cL(\vw^*) &\resizebox{.8\textwidth}{!}{$= \frac{1}{2}\langle \vH, \Ee{(\vx, y) \sim \cD}{\vomega_t \vomega_t^\s{T}}\rangle= \frac{1}{2}\langle \vH, \Ee{(\vx, y) \sim \cD}{(\vomega_t^{bias} + \vomega_t^{variance}) (\vomega_t^{bias} + \vomega_t^{variance})^\s{T}}\rangle$}\\
    &\leq \langle \vH, (\Ee{(\vx, y) \sim \cD}{\vomega_t^{bias} (\vomega_t^{bias})^\s{T}} + \Ee{(\vx, y) \sim \cD}{\vomega_t^{variance} (\vomega_t^{variance})^\s{T}}) \rangle\\
    &= 2\Big( \frac{1}{2}\langle \vH, \Ee{(\vx, y) \sim \cD}{\vomega_t^{bias} (\vomega_t^{bias})^\s{T}}\rangle + \frac{1}{2}\langle \vH, \Ee{(\vx, y) \sim \cD}{\vomega_t^{variance} (\vomega_t^{variance})^\s{T}} \rangle \Big)\\
    &= 2\Big( \underbrace{(\cL(\vw_t^{bias}) - \cL(\vw^*))}_{\text{Bias Term Risk}} + \underbrace{(\cL(\vw_t^{variance}) - \cL(\vw^*))}_{\text{Variance Term Risk}} \Big),\numberthis\label{eq:ambsgd-new-final-iterate-split}
\end{align*}
where, $\cL(\vw_t^{bias}) - \cL(\vw^*) := \frac{1}{2}\langle \vH, \Ee{(\vx, y) \sim \cD}{\vomega_t^{bias} (\vomega_t^{bias})^\s{T}}\rangle$, and \\  
    $\qquad \qquad \cL(\vw_t^{variance}) - \cL(\vw^*) := \frac{1}{2}\langle \vH, \Ee{(\vx, y) \sim \cD}{\vomega_t^{variance} (\vomega_t^{variance})^\s{T}}\rangle.$
    
Similarly, the overall error in the tail averaged iterate will be bounded as: 
\begin{align*}
    &\cL(\overline{\vw}) - \cL(\vw^*)= \frac{1}{2}\langle \vH, \Ee{(\vx, y) \sim \cD}{\overline{\vomega} \overline{\vomega}^\s{T}}\rangle\\
    &= \frac{1}{2}\Big\langle \vH, \Ee{(\vx, y) \sim \cD}{\left(\frac{2}{T}\sum_{t'=1}^{T/2}\vomega_{t'}\right) \left(\frac{2}{T}\sum_{t'=1}^{T/2}\vomega_{t'}\right)^\s{T}}\Big\rangle\leq \frac{4}{\eta T^2}\sum_{t'=1}^{T/2} \s{Tr}(\Ee{(\vx, y) \sim \cD}{\vomega_{t'} \vomega_{t'}^\s{T}})\\
    &= \frac{4}{\eta T^2}\sum_{t'=1}^{T/2} \s{Tr}(\Ee{(\vx, y) \sim \cD}{\vB_{t'-1}\vomega_{t'-1} \vomega_{t'-1}^\s{T}\vB_{t'-1}^\s{T} + \eta^2\vnu_{t'-1}\vnu_{t'-1}^\s{T}})\\
    &= \frac{4}{\eta T^2}\sum_{t'=1}^{T/2} \s{Tr}(\Ee{(\vx, y) \sim \cD}{\vB_{t-1}\vomega_{t-1} \vomega_{t-1}^\s{T}\vB_{t-1}^\s{T}} + \frac{\eta^2}{b}\vSigma + \frac{\eta^2\alpha^2}{b^2}\zeta^2\vI) \numberthis\label{eq:ambsgd-new-tail-bv-split}\\
    &\leq \frac{8}{\eta T^2}\sum_{t'=1}^{T/2}\s{Tr}(\Ee{(\vx, y) \sim \cD}{\vomega_{ t'}^{bias} (\vomega_{ t'}^{bias})^\s{T}} + \Ee{(\vx, y) \sim \cD}{\vomega_{ t'}^{variance} (\vomega_{ t'}^{variance})^\s{T}})\\
    &:= \underbrace{\cL(\overline{\vw}^{bias}) - \cL(\vw^*)}_{\text{Tail Averaged Bias Risk}} + \underbrace{\cL(\overline{\vw}^{variance}) - \cL(\vw^*)}_{\text{Tail Averaged Variance Risk}}, \numberthis\label{eq:ambsgd-new-tail-averaged-split}
\end{align*}
where in \eqref{eq:ambsgd-new-tail-bv-split} we have used Lemma~\ref{lemma:ambsgd-nutnut^T} and the discussion around \eqref{eq:ambsgd-new-nujnuj-zeta-inequality}, and
\begin{align*}
    \cL(\overline{\vw}^{bias}) - \cL(\vw^*) &:= \frac{8}{\eta T^2}\sum_{t'=1}^{T/2}\s{Tr}(\Ee{(\vx, y) \sim \cD}{\vomega_{ t'}^{bias} (\vomega_{ t'}^{bias})^\s{T}}),\\
    \cL(\overline{\vw}^{variance}) - \cL(\vw^*) &:= \frac{8}{\eta T^2}\sum_{t'=1}^{T/2}\s{Tr}(\Ee{(\vx, y) \sim \cD}{\vomega_{ t'}^{variance} (\vomega_{ t'}^{variance})^\s{T}}).
\end{align*}
\textbf{Bias Term Analysis:}\\
By virtue of the form of Bias Variance Decomposition in \eqref{eq:ambsgd-new-final-iterate-split} and \eqref{eq:ambsgd-new-tail-averaged-split}, the Bias Term error for the tail averaged iterate is the same as that of the non-private fixed mini-batch SGD case.

Therefore, using Lemmas 10 and 11 of \cite{jain2018parallelizing} in the form of Lemma ~\ref{lemma:parallel-bias-risk}, we get
\begin{align*}
    \cL(\overline{\vw}^{bias}) - \cL(\vw^*) &\leq \frac{8}{\eta^2 \mu^2 T^2}(1-\eta\mu)(\cL(\vw_{ 0}) - \cL(\vw^*))\leq \frac{8}{\eta^2\mu^2 T^2}e^{-\eta\mu}(\cL(\vw_{ 0}) - \cL(\vw^*)). \numberthis\label{eq:ambsgd-new-bias-tail-risk}
\end{align*}
\textbf{Variance Term Analysis:}\\
Note that the Variance Term denotes starting the initialising $\vw_{0}$ with the ground truth $\vw^*$ and letting the inherent noise $\vxi$ and additive DP Gaussian noise $\frac{2\zeta\alpha}{b} \vg$ to drive the updates. We can therefore write down the Variance Term $\vomega^{variance}$ updates as:
\begin{align*}
    \vomega^{variance}_t &= \vB_{t-1}\vomega^{variance}_{t-1} + \eta\vnu_{t-1}
\end{align*}
with $\vomega^{variance}_{0} = \vzero$.

Defining $\vPhi^{variance}_t := \Ee{(\vx, y) \sim \cD}{\vomega^{variance}_t( \vomega^{variance}_t)^\s{T}}$, we have
\begin{align*}
    \vPhi^{variance}_t &:= \Ee{(\vx, y) \sim \cD}{\vomega^{variance}_t(\vomega^{variance}_t)^\s{T}}\\
    &= \Ee{(\vx, y) \sim \cD}{(\vB_{t-1}\vomega^{variance}_{t-1} + \eta\vnu_{t-1}) (\vB_{t-1}\vomega^{variance}_{t-1} + \eta\vnu_{t-1})^\s{T}}\\
    &= \Ee{(\vx, y) \sim \cD}{\left(\vB_{t-1}\vomega^{variance}_{t-1}\right) \left(\vB_{t-1}\vomega^{variance}_{t-1}\right)^\s{T}} + \Ee{(\vx, y) \sim \cD}{\eta\vnu_{t-1} \left(\vB_{t-1}\vomega^{variance}_{t-1}\right)^\s{T}} \\
    &\quad + \Ee{(\vx, y) \sim \cD}{\left(\vB_{t-1}\vomega^{variance}_{t-1}\right) \eta\vnu_{t-1}^\s{T}} + \Ee{(\vx, y) \sim \cD}{\eta\vnu_{t-1} \eta\vnu_{t-1}^\s{T}}\\
    &= \Ee{(\vx, y) \sim \cD}{\left(\vB_{t-1}\vomega^{variance}_{t-1}\right)\left(\vB_{t-1}\vomega^{variance}_{t-1}\right)^\s{T}} + \vzero + \vzero + \Ee{(\vx, y) \sim \cD}{\eta\vnu_{t-1} \eta\vnu_{t-1}^\s{T}}\numberthis\label{eq:ambsgd-new-lemma-phi-term1}\\
    &= \Ee{(\vx, y) \sim \cD}{\vB_{t-1}\vPhi^{variance}_{t-1}\vB_{t-1}^\s{T}} + \eta^2\Big(\frac{\vSigma}{b} + \frac{4\zeta^2\alpha^2}{b^2} \vI\Big),
\end{align*}
where in \eqref{eq:ambsgd-new-lemma-phi-term1} we have again used the fact the the fact that $\Ee{(\vx, y) \sim \cD}{\vg_{t-1}} =\vzero$ has been sampled independently at each step and in the last step we have used Lemma~ \ref{lemma:ambsgd-nutnut^T}.

Using the definition of $\vB_{t-1}$ gives $\Ee{(\vx, y) \sim \cD}{\vB_{t-1}\vPhi^{variance}_{t-1}\vB_{t-1}^\s{T}}$
\begin{align*}
    &\resizebox{\textwidth}{!}{$=\Ee{(\vx, y) \sim \cD}{\Big(\vI - \frac{\eta}{b}\sum_{i=0}^{b-1} \vx_{\tau(t-1)+s+i} \vx_{\tau(t-1)+s+i}^\s{T}\Big)\vPhi^{variance}_{t-1}\Big(\vI - \frac{\eta}{b}\sum_{i=0}^{b-1} \vx_{\tau(t-1)+s+i}\vx_{\tau(t-1)+s+i}^\s{T}\Big)^\s{T}}$}\\
    &:= (\cI - \eta\cT_{b})\vPhi_{t-1}^{variance},
\end{align*}
where $\cT_{b}$ represents the linear operator:
\begin{align*}
    \cT_{b} := \cH_L + \cH_R - \frac{\eta}{b}\cM - \eta\frac{b-1}{b}\cH_L\cH_R
\end{align*}
and $\cH_L = \vH \otimes \vI$ and $\cH_R = \vI \otimes \vH$ represent the left and right multiplication linear operators respectively.

Continuing to expand the recursion, we get
\begin{align*}
    \vPhi_t^{variance} &= (\cI - \eta\cT_{b})\vPhi_{t-1}^{variance} + \eta^2\Big(\frac{\vSigma}{b} + \frac{4\zeta^2\alpha^2}{b^2}\vI\Big)\\
    &= \frac{\eta^2}{b}\Big(\sum_{k=0}^{t-1}(\cI - \eta\cT_{b})^k\Big)\Big(\vSigma + \frac{4\zeta^2\alpha^2}{b}\vI\Big).
\end{align*}
Furthermore, this sequence of covariances is non-decreasing w.r.t $t$:
\begin{align*}
    \vPhi_{t+1}^{variance} - \vPhi_t^{variance} = \frac{\eta^2}{b}\Ee{(\vx, y) \sim \cD}{\vQ_{1,t+1}\Big(\vSigma + \frac{4\zeta^2\alpha^2}{b} \vI\Big)\vQ_{1,t+1}^\s{T}} \geq 0.
\end{align*}
By virtue of the form of the Variance Potential Decomposition above, the Variance error at iteration $\ell$ for the final as well as tail averaged iterate is the same as of the non-private fixed mini-batch SGD case, except that we now have $\vSigma + \frac{4\zeta^2\alpha^2}{b} \vI$ in place of $\vSigma$ of \cite{jain2018parallelizing}.

Therefore the bound for the Variance Term follows almost similarly from the paper, except that $\vSigma$ gets replaced with $\vSigma + \frac{4\zeta^2\alpha^2}{b} \vI$. Using Lemmas 13, 14 and 15 of \cite{jain2018parallelizing} in the form of Lemma \ref{lemma:parallel-variance-risk}, we obtain the tail averaged iterate for our case of additive DP noise as
\begin{align*}
    \cL(\overline{\vw}^{variance}) - \cL(\vw^*) \leq \frac{4}{Tb}\s{Tr}\Big(\cT_b^{-1}\Big(\vSigma + \frac{4\zeta^2\alpha^2}{b} \vI\Big)\Big),
\end{align*}
such that if
\begin{align*}
    \vA &:= \Big(\cH_{\cL} + \cH_{\cR} - \frac{\eta}{b}\cdot(b-1)\cH_{\cL}\cH_{\cR}\Big)^{-1}\Big(\vSigma + \frac{4\zeta^2\alpha^2}{b} \vI\Big)\\
\end{align*}
and $\eta \leq \frac{2b}{R_x^2\cdot\frac{d\Big\|(\cH_{\cL} + \cH_{\cR})^{-1}\Big(\vSigma + \frac{4\zeta^2\alpha^2}{b} \vI\Big)\Big\|_2}{\s{Tr}\Big((\cH_{\cL} + \cH_{\cR})^{-1}\Big(\vSigma + \frac{4\zeta^2\alpha^2}{b} \vI\Big)\Big)}}$, then
\begin{align*}
    \s{Tr}\Big(\cT_{b}^{-1}\Big(\vSigma + \frac{4\zeta^2\alpha^2}{b}\vI\Big)\Big) &\leq \s{Tr}(\vA) + \frac{\frac{\eta R_x^2}{2b}d\Big\|(\cH_{\cL} + \cH_{\cR})^{-1}\Big(\vSigma + \frac{4\zeta^2\alpha^2}{b} \vI\Big)\Big\|_2}{\Big(1 - \frac{\eta}{2b}\cdot(R_x^2 + (b-1)\|\vH\|_2)\Big)\Big(1 - \eta\frac{b-1}{2b}\|\vH\|_2\Big)}\\
    &\leq 2\s{Tr}(\vH^{-1}(\vSigma + \frac{4\zeta^2\alpha^2}{b} \vI)).
\end{align*}
Thus we get:
\begin{align*}
    \cL(\overline{\vw}^{variance}) - \cL(\vw^*) &\leq \frac{8}{Tb}\Big(\s{Tr}(\vH^{-1}\vSigma) + \frac{4\zeta^2\alpha^2}{b} \s{Tr}(\vH^{-1})\Big).\numberthis\label{eq:ambsgd-new-variance-tail-risk}
\end{align*}
Combining the tail averaged iterate's Bias Error \eqref{eq:ambsgd-new-bias-tail-risk} and Variance Error \eqref{eq:ambsgd-new-variance-tail-risk} in \eqref{eq:ambsgd-new-tail-averaged-split}, we get that at for each iteration $\ell$, the overall risk for the tail averaged iterate as $\cL(\overline{\vw}) - \cL(\vw^*)$
\begin{align*}
    &= \underbrace{\cL(\overline{\vw}^{bias}) - \cL(\vw^*)}_{\text{Tail Averaged Bias Term Risk}} + \underbrace{\cL(\overline{\vw}^{variance}) - \cL(\vw^*)}_{\text{Tail Averaged Variance Term Risk}}\\
    &\leq \frac{8}{\eta^2\mu^2T^2}e^{-\eta\mu}(\cL(\vw_{0}) - \cL(\vw^*)) + \frac{8}{Tb}\Big(\s{Tr}(\vH^{-1}\vSigma) + \frac{4\zeta^2\alpha^2}{b} \s{Tr}(\vH^{-1})\Big)\numberthis\label{eq:ambsgd-new-tail-averaged-plugged-zeta-unplugged}
\end{align*}
Using the value of $\zeta^2$ from \eqref{eq:ambsgd-new-zeta-def} and initialization $\vw_0 = \vw_{T/2}$ for the second run in \eqref{eq:ambsgd-new-tail-averaged-plugged-zeta-unplugged} gives with probability $1-\frac{1}{\sf Poly(N)}$, $\cL(\overline{\vw}) - \cL(\vw^*)$
\begin{align*}
    &\leq \frac{4}{\eta^2\mu^2T^2}e^{-\eta\mu}\Ee{(\vx, y)\sim \cD}{\|\vw_{T/2} - \vw^*\|^2_{\vH}} + \frac{8}{Tb}\Big(\s{Tr}(\vH^{-1}\vSigma) \\
    &\quad + 3 K_2^2R_x^2\log^{4a} N\Big(\kappa\Ee{(\vx, y) \sim \cD}{\|\vw_{T/2}- \vw^*\|^2_{\vH}} + \sigma^2 + \Delta^2\Big) \cdot \frac{4\alpha^2}{b}\s{Tr}(\vH^{-1})\Big)\\
    &\leq \Big(\frac{4}{\eta^2\mu^2T^2}e^{-\eta\mu} + \frac{96\alpha^2}{Tb^2}K_2^2R_x^2\kappa\log^{4a} N\s{Tr}(\vH^{-1})\Big)\Ee{(\vx, y)\sim \cD}{\|\vw_{T/2} - \vw^*\|^2_{\vH}} \\
    &\quad + \frac{8}{Tb}\Big(\s{Tr}(\vH^{-1}\vSigma) + \frac{12\alpha^2}{b}K_2^2R_x^2\log^{4a}N(\sigma^2+\Delta^2)\s{Tr}(\vH^{-1})\Big).
    \numberthis\label{eq:mbsgd-new-tail-averaged-risk-plugged}
\end{align*}

Putting the value of $\Ee{(\vx, y) \sim \cD}{\|\vw_{T/2} - \vw^*\|^2_{\vH}} $ from \eqref{eq:ambsgd-new-wt/2-value} in above thus gives
\begin{align*}
    \cL(\overline{\vw}) - \cL(\vw^*)&\leq \Big(\frac{4}{\eta^2\mu^2T^2}e^{-\eta\mu} + \frac{96\alpha^2}{Tb^2}K_2^2R_x^2\kappa\log^{4a} N\s{Tr}(\vH^{-1})\Big)\cdot\\
    &\quad \Big(\frac{\|\vw^*\|^2_\vH}{N^{\frac{\eta\mu}{4}c_1\kappa}} + \frac{2\eta}{\mu b}\s{Tr}(\vH\vSigma) + \frac{24\eta\alpha^2}{\mu b^2} K_2^2R_x^2\log^{4a} N(\sigma^2 + \Delta^2) \s{Tr}(\vH)\Big) \\
    &\quad + \frac{8}{Tb}\Big(\s{Tr}(\vH^{-1}\vSigma) + \frac{12\alpha^2}{b}K_2^2R_x^2\log^{4a}N(\sigma^2+\Delta^2)\s{Tr}(\vH^{-1})\Big).
\end{align*}

which is the required result.
\end{proof}

\end{document}